\documentclass{article}

\usepackage{arxiv}

\usepackage[utf8]{inputenc} 
\usepackage[T1]{fontenc}    
\usepackage{hyperref}       
\usepackage{url}            
\usepackage{booktabs}       
\usepackage{amsfonts}       
\usepackage{nicefrac}       
\usepackage{microtype}      
\usepackage{lipsum}		
\usepackage{graphicx}
\usepackage{natbib}
\usepackage{doi}

\usepackage{lineno}

\usepackage{amsmath}
\usepackage{amssymb}
\usepackage{amstext}
\usepackage{amsthm}

\usepackage{lineno}
\modulolinenumbers[5]
\usepackage[font=small,labelfont=bf]{caption}
\usepackage{algorithm}
\usepackage{algpseudocode}

\usepackage{soul}
\usepackage[normalem]{ulem}
\usepackage{xcolor}
\newcommand\x{\bgroup\markoverwith{\textcolor{red}{\rule[0.5ex]{2pt}{0.4pt}}}\ULon}

\usepackage{booktabs}
\usepackage{multirow}

\usepackage{rotating}
\usepackage{array,multirow,graphicx}
\newcommand{\STAB}[1]{\begin{tabular}{@{}c@{}}#1\end{tabular}}

\newcommand{\EX}{\mathbb{E}}

\DeclareMathOperator*{\argmax}{arg\,max}
\DeclareMathOperator*{\argmin}{arg\,min}

\title{Genetic multi-armed bandits: a reinforcement learning approach for discrete optimization via simulation}

\author{ \href{https://orcid.org/0000-0001-7928-7267}{\includegraphics[scale=0.2]{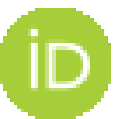}\hspace{1mm}Deniz Preil}\thanks{Corresponding author} \\
 	University of Augsburg\\
    Universitaetsstr. 16\\
	86159 Augsburg, Germany \\
	\texttt{deniz.preil@uni-a.de} \\
	\And
	\hspace{1mm}Michael Krapp \\
 	University of Augsburg\\
    Universitaetsstr. 16\\
	86159 Augsburg, Germany \\
	\texttt{michael.krapp@uni-a.de} \\
}

\renewcommand{\headeright}{}
\renewcommand{\undertitle}{ }
\renewcommand{\shorttitle}{Genetic multi-armed bandits}

\hypersetup{
pdftitle={Genetic multi-armed bandits: a reinforcement learning approach for discrete optimization via simulation},
pdfauthor={Deniz Preil},
pdfkeywords={simulation optimization, multi-armed bandits, genetic algorithms, reinforcement learning},
}

\begin{document}
\maketitle

\begin{abstract}
	This paper proposes a new algorithm, referred to as GMAB, that combines concepts from the reinforcement learning domain of multi-armed bandits and random search strategies from the domain of genetic algorithms to solve discrete stochastic optimization problems via simulation.
In particular, the focus is on noisy large-scale problems, which often involve a multitude of dimensions as well as multiple local optima.
Our aim is to combine the property of multi-armed bandits to cope with volatile simulation observations with the ability of genetic algorithms to handle high-dimensional solution spaces accompanied by an enormous number of feasible solutions.
For this purpose, a multi-armed bandit framework serves as a foundation, where each observed simulation
is incorporated into the memory of GMAB. Based on this memory, genetic operators guide the search, as they provide powerful tools for exploration as well as exploitation.
The empirical results demonstrate that GMAB achieves superior performance compared to benchmark algorithms from the literature in a large variety of test problems. In all experiments, GMAB required considerably fewer simulations to achieve similar or (far) better solutions than those generated by existing methods.
At the same time, GMAB's overhead with regard to the required runtime is extremely small due to the suggested tree-based implementation of its memory.
Furthermore, we prove its convergence to the set of global optima as the simulation effort goes to infinity.
\end{abstract}

\keywords{Simulation optimization \and multi-armed bandits \and genetic algorithms \and reinforcement learning}

\section{Introduction}

Real-world problems often involve various stochastic impact factors (such as customer demand, lead times, or production times) and complex system dynamics, which render the application of analytical methods infeasible. As closed-form expressions of the objective functions typically do not exist for such problems, simulation is used to find optimal solutions that provide the best system performance. This approach is known as simulation optimization \cite[]{amaran2016simulation} or synonymously referred to as optimization via simulation (OvS). OvS problems differ from deterministic problems and are generally difficult to solve. 
Due to the presence of noise in the underlying simulation model, one single simulation observation is insufficient to infer the true value of the corresponding solution. Instead, multiple simulations are required to obtain an adequate estimate of the true value. Furthermore, these simulations often require large computational effort.

In recent decades, numerous studies have focused on OvS problems, where the decision variables consist of integers, which are also referred to as discrete OvS (DOvS) problems. Such problems arise, to name just a few, when determining cost-minimal base-stock levels in multi-echelon supply chains or when maximizing the flow-line throughput by identifying an optimal buffer capacity allocation. 
Depending on the number of feasible solutions, the literature offers a variety of different DOvS methods. For a more detailed overview, see \cite{hong2015discrete}.
In our study, we primarily focus on problems with an enormous number of feasible solutions, which might be high-dimensional and multimodal as well. Methods tailored for such problems are divided into guaranteed convergence algorithms and those without convergence guarantees.

Among the former category, several approaches guarantee only local convergence, such as the algorithm of \cite{andradottir1995method}, 
the COMPASS variants of \cite{hong2006discrete} and \cite{hong2010speeding}, the AHA algorithm of \cite{xu2013adaptive}, and the R-SPLINE algorithm proposed by \cite{wang2013integer}. Locally convergent algorithms, however, are only able to identify local optima. While each of those local optima constitutes the best solution in a given predefined neighborhood, it is not necessarily the best of all possible solutions. Hence, the set of local optima and its size crucially depends on the definition of the neighborhood.

In contrast and as the name implies, globally convergent methods aim to find the best of all feasible solutions. Therefore, the search for optimal solutions focuses not only on specific neighborhoods but on the entire set of solutions. As this is usually more challenging, it requires a mechanism that quickly identifies the most promising areas. Examples of such algorithms are the stochastic ruler proposed by \cite{yan1992stochastic}, the global search method proposed by \cite{andradottir1996global}, the simulated annealing version of \cite{alrefaei1999simulated}, the nested partitions method of \cite{shi2000nested}, the SMRAS method suggested by \cite{hu2008model}, the BEESE framework of \cite{andradottir2009balanced}, the GPS algorithm of \cite{sun2014balancing} as well the approaches of \cite{l2019gaussian} and \cite{semelhago2021rapid} based on Gaussian Markov random fields.

While convergent methods are more prominent in academic research, no guaranteed convergent algorithms are often implemented in commercial software packages, such as OptQuest, or are employed in studies with a strong application focus. Examples include well-known metaheuristics such as genetic algorithms (GAs). These methods achieve very good results in deterministic problems even with a large number of dimensions and huge solution spaces. However, they often fail to provide comparably good solutions when noise is involved. Therefore, to handle stochastic problems, time-consuming methods, such as sample average approximation, are required \cite[]{jin2005evolutionary}. Nevertheless, even then finding an optimal solution is not guaranteed.
Despite the absence of a convergence guarantee in DOvS problems, GAs deploy powerful random search strategies such as crossover or mutation.

The central contribution of this paper is to propose a new algorithm, called genetic multi-armed bandit (GMAB), for high-dimensional DOvS problems with a huge number of feasible solutions. 
GMAB combines key concepts from the reinforcement learning domain of multi-armed bandits (MABs) and random search strategies from the domain of GAs. 
Genetic operators are pivotal, as they provide a powerful tool for exploring unvisited areas of the solution space as well as exploiting the neighborhood of promising solutions. Unlike traditional GAs, GMAB searches for solutions not only based on the population of chromosomes (or solution candidates) of the current iteration but also on all solutions visited over the course of iterations. Furthermore, each simulation observation is used as an update to obtain a more accurate estimate of the corresponding solution. 
These two (memory-related) features are common practice among MAB algorithms and are also required to ensure the global convergence of GMAB.
Generally, we combine the strength of GAs to be suitable for problems with extremely large solution spaces with the ability of MABs to cope with volatile simulation observations.
However, as the number of visited solutions grows, so does the complexity of operations, such as selecting which solutions to visit next, since each visited solution and its estimate need to be added to the memory of GMAB. We therefore propose a new time-efficient procedure based on two balanced binary trees that ensure logarithmic complexity with respect to GMAB's memory size.
In our paper, we show that GMAB converges with probability 1 to the set of globally optimal solutions as the simulation effort increases. We also analyze the finite-time performance of GMAB in a variety of test problems from the literature and compare it with those of other convergent DOvS algorithms. The results demonstrate that GMAB requires fewer simulations to achieve similar or even better solutions than those generated by existing methods.
At the same time, it requires less computational effort than most benchmark algorithms from the literature due to the tree-based implementation of its memory.
Even in challenging high-dimensional and multimodal problems, GMAB achieves remarkable performance.
Based on the results, we are convinced that GMAB should play a central role in solving DOvS problems.

The remainder of the paper is organized as follows. In Section \ref{sec:Background}, we briefly review the background of GAs and MABs in the context of DOvS problems and outline which concepts of GAs and MABs are incorporated into GMAB. A detailed description of GMAB is given in Section \ref{sec:Genetic Multi-Armed Bandits} together with the proof of global convergence, the proposed stopping as well as the final selection criterion, and the tree-based implementation of memory. Section \ref{sec:Numerical Experiments} reports the numerical results, while Section \ref{sec:Conclusion} provides the conclusion.

\section{Background}
\label{sec:Background}

Since GMAB adopts principles from both MABs and GAs, 
we briefly introduce their fundamentals against the backdrop of DOvS problems in the following.

MABs are particularly suited for problems with a comparatively small number of feasible solutions. 
For a comprehensive overview, we refer to \cite{bubeck2012regret} and \cite{lattimore2020bandit}.
The name multi-armed bandit is inspired by the fictitious problem of a gambler aiming to identify the best slot machine. By sequentially pulling arms, the gambler receives rewards, which are in turn used to update the estimated values of the arms. In this nomenclature, an arm corresponds to a feasible solution and a reward is equal to a simulation observation.
While one family of MAB algorithms focuses on the maximization of the total reward received over a certain number of simulations, the other aims to find the arm that generates the largest reward once a stopping criterion, such as a certain number of simulation replications, is met. The latter problem is also referred to as the `pure exploration problem' or the best arm identification (BAI) problem \cite[]{bubeck2012regret}. 
In the remainder of this paper, when we speak of MAB algorithms, the sole focus is on MABs solving the BAI problem.
A main characteristic of MABs is the sequential pulling of arms and the consequential `learning' about the true values based on the received rewards. In this regard, the exploration-exploitation trade-off occurs. Is it advisable to choose arms other than those currently considered to be promising in order to discover ones that may be even better (`exploration')? Or is it preferable to continue selecting promising arms to get more accurate estimates of their true values (`exploitation')?
To tackle this trade-off, MABs store all rewards received thus far. Often this storage is accomplished by the incremental update of an arms' sample mean or by updating an arms' prior distribution in a Bayesian setting. In addition, the decision on which arm to pull next is typically made based on all information about the rewards gathered thus far. We will refer to this feature in the following as `full memory'. It is a major property of the GMAB algorithm presented in Section \ref{sec:Genetic Multi-Armed Bandits}.
Full memory is not unique to MABs. It is, for example, also part of many other DOvS algorithms including ranking and selection (R\&S) methods. For an overview of the latter, we refer to \cite{hong2021review}.
Note that the streams of research on R\&S and MABs (with regard to BAI problems) are closely related. 
They have the same roots but subsequently evolved in different communities. Nevertheless, they share the same objective \cite[]{hong2021review}.
It is therefore not possible to conclusively specify which family of methods the presented algorithm is more related to. In our opinion, the relation with MABs is larger, since `sequentially' visiting solutions and updating their estimated values is more predominant among MABs (although there is a long stream of literature on sequential R\&S as well). In addition, we use a final selection criterion whose roots stem clearly from the MAB domain.
MABs (as well as R\&S approaches) are especially designed for stochastic problems. However, they are usually only applicable to problems with a small number of feasible solutions since each solution usually needs to be visited at least once. In DOvS problems with a large number of feasible solutions, this is not possible. 

The exact opposite applies to GAs. 
They are tailored to large solution spaces but (in their primitive form) tend to perform worse as noise increases.
GAs adopt the principles of biological evolution and randomly search for a solution based on a population of candidates \cite[]{sivanandam2008genetic}. Within a population, the candidates are evaluated according to their fitness values, whereby candidates with larger fitness values are more likely to reproduce.
One major strength of GAs is their powerful random search strategies, such as crossover or mutation, which, for this reason, are also a central part of GMAB.
Despite being widely used in environments with moderate noise, GAs were originally developed for deterministic settings (usually with a high-dimensional and therefore huge solution space). Consequently, they do not provide any optimality guarantee in DOvS problems.
If noise is involved, one naïve approach is to assume a self-averaging nature, that is, to hope that promising solutions will prevail as the number of iterations increases. A commonly used alternative is to employ variance reduction techniques such as average approximation. However, this only reduces the volatility of the fitness values but does not eliminate it. For further techniques to modify GAs for stochasticity, we refer to \cite{jin2005evolutionary}.

As we aim to solve high-dimensional DOvS problems with large noise, the main idea of the proposed GMAB algorithm is to adapt the properties of 
\begin{itemize}
\item MABs to cope with volatile simulation observations and 
\item GAs to handle high-dimensional solution spaces with an enormous number of feasible solutions. 
\end{itemize}
The first core element of GMAB can be thought of as a traditional MAB framework, in which each simulation observation is stored and utilized to determine which solutions to visit next.
The second element of GMAB comprises the genetic operators that specify this determination. 
Unlike conventional GAs, there is no single population but rather the entire memory of solutions visited thus far.
Employing genetic operators that ensure each solution is able to generate any offspring solution from the set of all feasible solutions with nonzero probability, GMAB converges to an optimal solution with probability 1 as the computational budget goes to infinity.
Note that there have already been (few) attempts to combine principals of MABs with evolutionary strategies, cf. \cite{liu2017bandit}, \cite{lucas2018n}, and \cite{qiu2019enhancing}.
However, these approaches differ fundamentally from GMAB. 
For example, they provide no convergence guarantee, they lack full memory, and they are designed for rather small-size problem instances.

\section{Genetic Multi-Armed Bandits}
\label{sec:Genetic Multi-Armed Bandits}

In the following we consider a DOvS problem of the form
\begin{equation} 
\begin{array}{cl}
\min &g(\textbf{x})=\EX[G(\textbf{x})] \\
 & \\
\text{s.t.} &\textbf{x}\in \Theta = \Omega \cap \mathcal{Z}^{D},
\end{array}
\label{f1}
\end{equation} 
where $\Omega$ is a
closed and bounded set and $\mathcal{Z}^{D}$ is the $D$-dimensional integer lattice.
The distribution of the random variable $G(\textbf{x})$ is an unknown function of the vector-valued decision variable $\textbf{x}=(x_1,\dots,x_D)^{\mathrm{T}}$. In this regard, we use $x_d^\mathrm{LB}$ and $ x_d^\mathrm{UB}$, respectively, to represent the lower and upper bounds of $x_d$ for all $d\in\{1,\dots,D\}$.
Although there is no closed-form expression of $G(\textbf{x})$, it is possible to observe a realization by performing a (mostly computationally expensive) simulation experiment at $\textbf{x}$. 
Let $\overline{G}(\textbf{x})$ denote the sample mean of observations of $\textbf{x}$. 
Note that in most relevant applications, it is a strongly consistent estimator of $g(\textbf{x})=\EX[G(\textbf{x})]$.
Furthermore, we use $\textbf{x}^*$ to denote a solution from the non-empty set of (globally) optimal solutions $\Theta^*$ and $\hat{\textbf{x}}^*_k$ to represent 
the solution that the algorithm considers to be the best if it would have stopped at the end of iteration $k$.

\subsection{The GMAB Algorithm}
Let $\mathcal{V}_k$ be the set of all solutions visited up to iteration $k$,  $N_k(\textbf{x})$ the number of all observations of solution $\textbf{x}$ up to $k$, and $R_k(\textbf{x})=\sum_{j=0}^k I_j(\textbf{x})$ the sum of all observations of solution \textbf{x} up to $k$, where:
\begin{equation} 
I_k(\textbf{x})=\begin{cases}             
G_k(\textbf{x})             ,  \quad \hspace{10pt}   \text{if } \textbf{x} \text{ was simulated in iteration } k\\
0        ,  \quad  \hspace{31pt}    \text{otherwise.}  
\end{cases}
\label{f2}
\end{equation} 
$G_k(\textbf{x})$ denotes the observation of $G(\textbf{x})$ received in iteration $k$ (if $\textbf{x}$ was simulated in $k$). 
Furthermore, let $\overline{G}_k(\textbf{x})=R_k(\textbf{x})/N_k(\textbf{x})$ be the sample mean of solution $\textbf{x}$ at $k$ $\forall \textbf{x} \in \mathcal{V}_k$.
In every iteration, the algorithm selects a set of solutions $\mathcal{S}_k \subseteq \Theta$ that will be visited in $k$. Since  $\mathcal{S}_k$ is a set, we do not allow duplicates.
If $k=0$, $\mathcal{S}_k$ consists of $m$ randomly selected solutions, where $m < |\Theta|$ is a natural even number.
If $k>0$, then $\mathcal{S}_k=\mathcal{E}_k\cup\mathcal{M}_k$.
Here, $\mathcal{E}_k$ denotes the best $m$ solutions out of $\mathcal{V}_{k-1}$, i.e., those with the smallest $\overline{G}_{k-1}(\textbf{x})$, 
and $\mathcal{M}_k$ denotes the result of genetic operators modifying $\mathcal{E}_k$.
If the selection of the best $m$ solutions is not unique, chance decides which of the best solutions will be included in $\mathcal{E}_k$.
It is important to note that in iteration $k$, the GMAB algorithm not only visits the solutions in $\mathcal{M}_k$ but also `revisits' those in $\mathcal{E}_k$ to obtain a more reliable sample mean of the current best solutions.
Revisiting (promising) solutions is a strategy which is also part of other DOvS algorithms, such as those of  \cite{hong2006discrete} and \cite{xu2013adaptive}. 
However, instead of revisiting each $\textbf{x}\in\mathcal{E}_k$, they revisit each $\textbf{x}\in\mathcal{V}_k$  in every iteration which causes a large effort.
In the supplementary material we provide a comprehensive discussion about the choice of $\mathcal{E}_k$.
Furthermore, unlike in conventional GAs, each visited solution `survives', i.e., it is captured in $\mathcal{V}_k$, and the corresponding sample mean (in terms of $N_k(\textbf{x})$ and $R_k(\textbf{x})$) is stored.
That is, we implement full memory.
Consequently, not only the currently visited solutions determine the offspring that will be visited in the next iteration but also all solutions that have been visited thus far.
Depending on whether an infinite or a finite computational budget is assumed, GMAB 
never stops visiting solutions or it terminates after a certain stopping criterion is met.
The budget also impacts the final selection criterion, cf. Line 14 of Algorithm \ref{algo:GMAB}. 
It determines which solution among $\mathcal{V}_k$ is considered to be the best at iteration $k$.
Note that it is actually not necessary to compute $\hat{\textbf{x}}^*_k$ after each iteration, as it has no impact on the behavior of GMAB in subsequent iterations. 
Instead, we recommend determining $\hat{\textbf{x}}^*_k$ only once after the algorithm has stopped.
We investigate the infinite budget performance (i.e., if $k \to \infty$) in the next section and discuss the stopping criterion for the finite case in Section \ref{sec:Stopping Criterion}. 
However, we first describe the genetic operators that modify the current best $m$ solutions $\mathcal{E}_k$ in each iteration, cf. Line 8 of Algorithm \ref{algo:GMAB}.

\begin{algorithm}
  \begin{algorithmic}[1]
	\State{$k\gets 0$}
	\State{$\mathcal{S}_0 \gets $Randomly select $m$ distinct solutions out of $\Theta$}
		\State{$N_0(\textbf{x})\gets 1, \quad \hspace{70pt} R_0(\textbf{x})\gets G_0(\textbf{x}) \quad \hspace{53pt} \forall \textbf{x} \in \mathcal{S}_0$}
		\State{$\mathcal{V}_0\gets \mathcal{S}_0$}
		\State{$k\gets k+1$}
		\While{Stopping criterion is not met}
        \State{$\mathcal{E}_k\gets$ Select the best $m$ solutions out of $\mathcal{V}_{k-1}$}
        \State{$\mathcal{M}_k\gets$ \Call{GeneticModification}{$\mathcal{E}_{k}$}}
				\State{$N_{k-1}(\textbf{x})\gets 0, \quad \hspace{43pt} R_{k-1}(\textbf{x})\gets 0 \quad \hspace{63pt} \forall \textbf{x}\in \mathcal{M}_k \setminus \mathcal{V}_{k-1}$}
				\State{$\mathcal{S}_k\gets \mathcal{E}_{k} \cup \mathcal{M}_k$}
								\State{$N_k(\textbf{x})\gets N_{k-1}(\textbf{x})+1, \quad \hspace{5pt} R_k(\textbf{x})\gets R_{k-1}(\textbf{x})+G_k(\textbf{x}) \quad \hspace{5pt}\forall \textbf{x} \in \mathcal{S}_k$}
				\State{$N_k(\textbf{x})\gets N_{k-1}(\textbf{x}), \quad \hspace{5pt}\hspace{18pt} R_k(\textbf{x})\gets R_{k-1}(\textbf{x}) \quad \hspace{42pt}\forall \textbf{x} \in \mathcal{V}_{k-1} \setminus \mathcal{S}_k$}
				\State{$\mathcal{V}_k\gets \mathcal{S}_k \cup\mathcal{V}_{k-1}$}
				\State{$\hat{\textbf{x}}^*_k \gets$ \Call{FinalSelection}{$\mathcal{V}_{k}$}}	
				\State{$k\gets k+1$}
     \EndWhile
	\State\Return {$\hat{\textbf{x}}^*_{k-1}$}
 \caption{Genetic Multi-Armed Bandit}
\label{algo:GMAB}
\end{algorithmic}
\end{algorithm}

Genetic operators provide a powerful tool especially for exploration but also for exploitation. 
In this regard, the current best $m$ solutions are passed to the function \textsc{GeneticModification}($\mathcal{E}_k$). Subsequently, all solutions of $\mathcal{E}_k$ are randomly arranged in $m/2$ pairs. 
For each pair, a single point crossover is executed with the crossover probability $p_\mathrm{cr}$. This crossover takes place at random position $g$. To demonstrate the principle of a single-point crossover, consider two solutions $\textbf{x}_1=(x_{1,1},\dots, x_{1,D})^{\mathrm{T}}$ and $\textbf{x}_2=(x_{2,1},\dots, x_{2,D})^{\mathrm{T}}$ of an arbitrary pair and, for example, $g=1$. 
After a successful crossover, two offspring solutions  $\textbf{x}^{'}_1=(x_{1,1}, x_{2,2},\dots, x_{2,D})^{\mathrm{T}}$ and  $\textbf{x}^{'}_2=(x_{2,1},x_{1,2},\dots, x_{1,D})^{\mathrm{T}}$ are obtained.
With probability $1-p_\mathrm{cr}$ no crossover is executed and the solutions of the corresponding pair remain unchanged. In this case, $\textbf{x}^{'}_1=\textbf{x}_1$ and $\textbf{x}^{'}_2=\textbf{x}_2$. Regardless of whether the crossover was executed or not, for every pair $\textbf{x}_1$ and $\textbf{x}_2$, the offspring solutions $\textbf{x}^{'}_1$ and  $\textbf{x}^{'}_2$  are assigned to the mutation set $\mathcal{O}$, which thereafter contains $m$ solutions.

In the subsequent mutation step, each of the $D$ components of each solution in $\mathcal{O}$, i.e., $x_{i,d} \hspace{5pt}\forall i=1,\dots,m;  \hspace{10pt} \forall \hspace{1pt}  d=1,\dots,D$ mutates with probability $p_\mathrm{mu}$.  
For this purpose, we use a Gaussian mutation operator.
Let $u_{i,d}$ denote a random number sampled from a uniform distribution between 0 and 1 and $n_{i,d}$ a random number sampled from a normal distribution with mean 0 and standard deviation $\sigma_d$.
Therefore, for every $i=1,\dots,m$ and $d=1,\dots,D$:
\begin{equation}
x_{i,d}^{'}=\begin{cases}             
\mathrm{round}(x_{i,d}+n_{i,d})              ,  \quad \hspace{10pt}   \text{if } u_{i,d}\leq p_\mathrm{mu} \\
x_{i,d}             ,  \quad  \hspace{72pt}    \text{otherwise.}  
\end{cases}
\label{f3}
\end{equation} 
If $x_{i,d}^{'}$ is outside the domain of $x_d$, i.e. if $x_{i,d}^{'}$ is larger than the upper bound $x_d^\mathrm{UB}$ or smaller than the lower bound $x_d^\mathrm{LB}$, we sample a new Gaussian noise term and replace the old one. We repeat this process until $x_d^\mathrm{LB} \leq x_{i,d}^{'}\leq x_d^\mathrm{UB}$.
During this process, we always use $\sigma_d=0.1(x_d^\mathrm{UB}-x_d^\mathrm{LB})$, as suggested by the literature \cite[]{hinterding1995gaussian}.
Finally, \textsc{GeneticModification}($\mathcal{E}_k$) returns the set of $m$ offspring solutions generated by the Gaussian mutation operator.
For $p_\mathrm{mu}>0$, the proposed procedure ensures that each solution $\textbf{x}=(x_1,\dots,x_D)^{\mathrm{T}}\in\mathcal{E}_k$ is able to generate any offspring solution $\textbf{x}^{'}=(x^{'}_1,\dots,x^{'}_D)^{\mathrm{T}}$ from the set of feasible solutions $\Theta$ with nonzero probability. This characteristic is of central importance for global convergence, as it allows to escape from local optima.

In general, a high level of exploration is desirable at the beginning of the search since it enables the quick identification of promising regions. As the number of iterations increases, enhanced exploitation is important to primarily focus on these promising areas. In early iterations, the crossover operator provides a powerful tool for exploration. Due to the heterogeneity of solutions in $\mathcal{E}_k$, the single point crossover implies large jumps within the solution space.
Since mutation exploits an increasing number of comparably good solutions in the neighborhood of currently promising ones, the distances between solutions in $\mathcal{E}_k$ decrease. Hence, the jumps within the solution space also decrease, which in turn encourages exploitation.

\subsection{Convergence of GMAB}
\label{sec:Convergence of GMAB}

In this section, we investigate the performance of GMAB when the budget is infinite, i.e., when $k\to\infty$. Note that there is no stopping criterion in this case. We term an algorithm globally convergent if the infinite sequence $(\hat{\mathbf x}_k^*)_{k\in\mathbb N}$ of solutions considered to be optimal converges with probability 1 to the set of optimal solutions $\Theta^*$, that is if $\mathrm P[\lim\limits_{k\to\infty}\hat{\mathbf x}_k^*\in\Theta^*]=1$. To assure global convergence of GMAB, we make the assumption that $\bar G(\mathbf x)$ is a strongly consistent estimator of $g(\mathbf x)=\EX[G(\mathbf{x})]$:

\newtheorem{assumption}{Assumption}
\begin{assumption}\label{A:1}
	If Assumption \ref{A:1} is satisfied, GMAB is globally convergent.
\end{assumption}

%
$\bar G_k(\mathbf x)$ then converges almost surely towards $g(\mathbf x)$ for every $\mathbf x\in\Theta$. This assumption comprises the strong law of large numbers and the ergodic theorem as special cases. Most simulation output satisfies this assumption \cite[]{hong2006discrete}. Further note that GMAB applies a Gaussian mutation operator. As long as $p_\mathrm{mu}>0$, this operator ensures that for every solution $\mathbf x\in\mathcal E_k$ in each iteration $k$ there is a non-zero probability to generate any offspring solution $\mathbf x'\in\Theta$. We can now state the following theorem.
\newtheorem{theorem}{Theorem}
\begin{theorem}\label{Th:1}
	If Assumption \ref{A:1} is satisfied, GMAB is globally convergent.
\end{theorem}
%

\begin{proof}[Proof of Theorem \ref{Th:1}]

	First, note that because $\Theta$ is a finite set, the Gaussian mutation operator assures $\mathrm P[\lim\limits_{k\to\infty}\mathcal V_k=\Theta]=1$ as long as $p_\mu>0$. Therefore, every solution $\mathbf x\in\Theta$ is visited infinitely often with probability 1 as $k\to\infty$. Furthermore, note that $\mathrm{P}[\lim\limits_{k\to\infty}\hat{\mathbf x}_k^*\in\Theta^*]=1$ is equivalent to
	\begin{equation}\label{E:1}
		\mathrm P[c(|g(\hat{\mathbf x}_k^*)-\min_{\mathbf x\in\Theta}g(\mathbf x)|\geq\epsilon)_{k\in\mathbb N}=\infty]=0
	\end{equation}
	for any $\epsilon>0$, where $c(A_k)_{k\in\mathbb N}:=\sum\limits_{k\in\mathbb N}\mathbf1_{A_k}$ counts the number of events $A_k$ that occur within an infinite sequence of events $(A_k)_{k\in\mathbb N}$. For the left-hand side of (\ref{E:1})
	\begin{equation}\label{E:2}
		\renewcommand\arraystretch{1.2}
		\begin{array}{r@{\,\,}c@{\,\,}l}
			\mathrm P[c(|g(\hat{\mathbf x}_k^*)-\min\limits_{\mathbf x\in\Theta}g(\mathbf x)|\geq\epsilon)_{k\in\mathbb N}=\infty]&\leq&
				\underbrace{\mathrm P[c(|g(\hat{\mathbf x}_k^*)-\bar G_k(\hat{\mathbf x}_k^*)|\geq0.5\epsilon)_{k\in\mathbb N}=\infty]}_{=:\,\alpha}+\\
				&&\underbrace{\mathrm P[c(|\bar G_k(\hat{\mathbf x}_k^*)-\min\limits_{\mathbf x\in\Theta}g(\mathbf x)|\geq0.5\epsilon)_{k\in\mathbb N}=\infty]}_{=:\,\beta}
		\end{array}
	\end{equation}
	is valid due to the triangle inequality. Since $\alpha\leq\mathrm P[\exists\mathbf x\in\Theta:\,c(|\bar G_k(\mathbf x)-g(\mathbf x)|\geq0.5\epsilon)_{k\in\mathbb N}=\infty]$, applying Boole's inequality reveals
	\begin{equation*}
		\alpha\leq\sum_{\mathbf x\in\Theta}\mathrm P[c(|\bar G_k(\mathbf x)-g(\mathbf x)|\geq0.5\epsilon)_{k\in\mathbb N}=\infty],
	\end{equation*}
	which is zero with probability 1 for any $\epsilon>0$ due to Assumption \ref{A:1} in conjunction with $|\Theta|<\infty$.

	\sloppy Regarding the probability $\beta$ in (\ref{E:2}), note that the selection criterion $\hat{\mathbf x}_k^*=\argmin\{\bar G_k(\mathbf x)|\mathbf x\in\mathcal V_k\}$ in combination with $\mathrm P[\lim\limits_{k\to\infty}\mathcal V_k=\Theta]=1$ implies $\bar G_k(\hat{\mathbf x}_k^*)=\min\limits_{\mathbf x\in\Theta}\bar G_k(\mathbf x)$ with probability 1 for sufficiently large $k$, i.e. for all $k$ such that $\mathcal V_k=\Theta$. Now suppose sample path realizations with $\min\limits_{\mathbf x\in\Theta}\bar G_k(\mathbf x)\geq\min\limits_{\mathbf x\in\Theta}g(\mathbf x)$. Then
	\begin{equation*}
		|\min\limits_{\mathbf x\in\Theta}\bar G_k(\mathbf x)-\min\limits_{\mathbf x\in\Theta}g(\mathbf x)|\leq
		\bar G_k[\argmin\limits_{\mathbf x\in\Theta}g(\mathbf x)]-\min\limits_{\mathbf x\in\Theta}g(\mathbf x)\leq
		\max\limits_{\mathbf x\in\Theta}|\bar G_k(\mathbf x)-g(\mathbf x)|.
	\end{equation*}
	The same is true for all sample path realizations with $\min\limits_{\mathbf x\in\Theta}\bar G_k(\mathbf x)\leq\min\limits_{\mathbf x\in\Theta}g(\mathbf x)$, since then
	\begin{equation*}
    \setlength{\arraycolsep}{0pt}
		\renewcommand\arraystretch{1.0}
		\begin{array}{r@{\,\,}c@{\,\,}l}
			|\kern-0.15em\min\limits_{\mathbf x\in\Theta}\bar G_k(\mathbf x)\kern-0.15em-\kern-0.15em\min\limits_{\mathbf x\in\Theta}g(\mathbf x)|\kern-0.15em=\kern-0.15em
			\min\limits_{\mathbf x\in\Theta}g(\mathbf x)\kern-0.15em-\kern-0.15em\min\limits_{\mathbf x\in\Theta}\bar G_k(\mathbf x)&\kern-0.15em\leq\kern-0.15em&
			g[\argmin\limits_{\mathbf x\in\Theta}\bar G_k(\mathbf x)]-\min\limits_{\mathbf x\in\Theta}\bar G_k(\mathbf x)\\ &\kern-0.15em\leq\kern-0.15em&
			\max\limits_{\mathbf x\in\Theta}|g(\mathbf x)-\bar G_k(\mathbf x)|\kern-0.15em=\kern-0.15em\max\limits_{\mathbf x\in\Theta}|\bar G_k(\mathbf x)-g(\mathbf x)|.
		\end{array}
	\end{equation*}
	Hence,
	\begin{equation*}
		\beta\leq
		\mathrm P[c(\max\limits_{\mathbf x\in\Theta}|\bar G_k(\mathbf x)-g(\mathbf x)|\geq0.5\epsilon)_{k\in\mathbb N}=\infty]
		\kern-0.15em\leq\kern-0.15em
		\mathrm P[\exists x\in\Theta:\kern-0.15em\,c(|\bar G_k(\mathbf x)-g(\mathbf x)|\geq0.5\epsilon)_{k\in\mathbb N}=\kern-0.15em\infty]
	\end{equation*}
	follows, which, for the same reasons and under the same conditions as for $\alpha$, is zero with probability 1 for any $\epsilon>0$. This completes the proof.
\end{proof}
Note that the logic of this proof is closely related to the proof of Theorem 1 in \citet{hong2006discrete}. However, \citet{hong2006discrete} only provide a local convergence result. GMAB instead assures global convergence. This is driven by the fact that GMAB is designed to sample all solutions with non-zero probability. Further note that Theorem \ref{Th:1} implies that suboptimal solutions are visited infinitely often with probability 0 as $k\to\infty$. Thus, GMAB considers such solutions $\hat{\mathbf x}_k^*$ with probability 1 as optimal only finitely many times, although the budget of iterations is infinite.

Of course, it is desirable to ensure that the algorithm identifies an optimal solution given an infinite budget of iterations.
However, global convergence does not necessarily imply that the algorithm will achieve a `good' performance in finite time.
Since performance is the most important aspect and practical problems accompany a finite budget, we address the finite case in more detail in the subsequent sections.

\subsection{Stopping Criterion}
\label{sec:Stopping Criterion}

In practical applications, it is important to specify when a DOvS algorithm should stop.
For this purpose, the literature suggests various stopping criteria, whereas each stopping criterion entails different benefits and drawbacks.
For instance, one might simply stop if $\hat{\textbf{x}}^*_k$ does not change over several iterations.
However, this may be achieved too early.
Some other criteria aim to infer about the optimality gap, cf. \cite{sun2014balancing} and \cite{l2019gaussian}.
Since visiting all $ \textbf{x} \in\Theta$ is impossible, they model the unknown objective function surface as a realization of a random process.
For example,  \cite{l2019gaussian} employed Gaussian Markov random fields for this purpose.
However, such approach is accompanied by considerable computational effort as the number of dimensions in a DOvS problem increases \cite[]{semelhago2021rapid}.
Consequently, an application in higher-dimensional problems like TP4\_D10 to TP4\_D20 (see Section \ref{sec:Test Problems}) is practically infeasible due to excessive computational effort. Even problems such as TP2 with (see also Section \ref{sec:Test Problems}) would require a tremendous computational overhead. 

Against the background of the above 'curse of dimensionality` as well as of the fact that many practitioners regard the runtime as their main bottleneck
, we adopt a so-called fixed-budget stopping criterion which is applied not only in DOvS algorithms facing large solution spaces but also in R\&S and MAB approaches dealing with a much smaller number of feasible solutions \cite[]{hong2021review, kaufmann2016complexity}. Accordingly, the algorithm terminates after a certain budget, such as computation time, number of iterations, or number of simulation observations, has expired. 
Furthermore, a finite number of simulation observations or runtime (e.g., in seconds) offers the possibility to compare different DOvS algorithms in a fair and objective manner.

Keep in mind that each stopping criterion entails its own benefits and drawbacks. Consequently, there is no `best' stopping criterion and its selection strongly depends on the authors' perspective.
Furthermore, note that although the criterion terminates the search, it does not affect the general finite performance. The latter is characterized by the number of simulations required to achieve a certain solution quality.
Since the structure of $g(\textbf{x})$ is typically unknown, it requires numerical experiments to investigate such finite performance, especially in order to compare the performances of different DOvS algorithms.
For this purpose, we investigate GMAB's finite performance in Section \ref{sec:Numerical Experiments} and demonstrate how it outperforms previous DOvS algorithms in variety of existing test problems.
Before, however, we elaborate to what extent the final selection criterion in a finite setting differs from the infinite setting.

\subsection{Final Selection Criterion in the Finite Setting}
\label{sec:FSC}

In contrast to the infinite setting, a finite runtime raises the question of what final selection criterion should be used.
Of course, this criterion only needs to be applied once at the end of the last iteration, which is denoted by $K$ in the following.
One approach is to adopt the same criterion as for infinite runtime, namely, the solution with the best sample mean, i.e., $\hat{\textbf{x}}^*_K=\argmin \{\overline{G}_K(\textbf{x}) |\textbf{x}\in\mathcal{V}_K\}$. 
However, doing so entails risks. Consider the algorithm terminates after iteration $K$. For a (suboptimal) solution that is visited for the first time in $K$, namely, $\textbf{x}\in \mathcal{M}_K \setminus \mathcal{V}_{K-1}$, one might unfortunately observe an observation $G_K(\textbf{x})$ that is better than the sample means of all other previously visited solutions. Consequently, one would favor this barely visited solution over other solutions with much more robust sample means.
A different approach is to select the solution with the greatest number of simulation observations, i.e., $\textbf{z}_K=\argmax \{N_K(\textbf{x}) |\textbf{x}\in\mathcal{V}_K\}$, as $\hat{\textbf{x}}^*_K$. 
Nevertheless, this also entails a drawback. Consider an actual second-best solution is visited, for example, in the first iteration and then in each subsequent iteration (as it is contained in every $\mathcal{E}_k$). An actual best solution will then never be identified as the best, unless it is also already present in the first iteration and in all subsequent ones.

To consider both aspects, $\hat{\textbf{x}}^*_K$ should be chosen at least from the set of nondominated solutions $\mathcal{J}_K=\{\textbf{x}|  \overline{G}_K(\textbf{x}) \leq \overline{G}_K(\textbf{z}_K) \land \textbf{x}\in\mathcal{V}_K \}$.
If $|\mathcal{J}_K|=1$, we simply choose $\hat{\textbf{x}}^*_K=\mathcal{J}_K$.
Otherwise, we propose the following selection criterion.
It combines the above criteria to a hybrid one that is inspired by the upper confidence bound (UCB) criterion widespread in the MAB domain, see \cite{auer2002finite} for more details. 
UCB is based on the confidence interval $[\overline{G}_K(\textbf{x})-U_K(\textbf{x}), \overline{G}_K(\textbf{x})+U_K(\textbf{x})]$ with size $2U_K(\textbf{x})$. In our case, this interval is only determined for each $\textbf{x}\in\mathcal{J}_K$, since the solutions $\textbf{x}\in\mathcal{V}_K\setminus \mathcal{J}_K$ are not considered anyway.
According to Hoeffding's inequality (which, however, assumes $G(\textbf{x})$ to be bounded between 0 and 1), for each $\textbf{x}\in\mathcal{J}_k$:
\begin{equation}
\mathrm{P}[\overline{G}_K(\textbf{x})-\EX[G(\textbf{x})]\geq U_K(\textbf{x})] \leq \mathrm{exp}^{-2U_K(\textbf{x})^2N_K(\textbf{x})}.
\label{f4}
\end{equation} 
Consequently, $U_K(\textbf{x})=\sqrt{-\mathrm{ln}(\delta)/(2N_K(\textbf{x}))}$, where $\delta$ is the right-hand side of (\ref{f4}).
In their seminal paper, \cite{auer2002finite} proposed $\delta=[\sum_{\textbf{y} \in \mathcal{V}_K} N_K(\textbf{y})]^{-4}$ , so  $\delta$ decreases with an increasing total number of simulations.
Hence, $U_K(\textbf{x})=\sqrt{2\mathrm{ln}(\sum_{\textbf{y} \in \mathcal{V}_K} N_K(\textbf{y}))/N_K(\textbf{x})}$.
In spirit of the `optimism in the face of uncertainty' principle, UCB chooses the solution $\textbf{x}\in\mathcal{J}_K$ that minimizes $\overline{G}_K(\textbf{x})-U_K(\textbf{x})$ (or maximizes $\overline{G}_K(\textbf{x})+U_K(\textbf{x})$ in a maximization problem).
The term $U_K(\textbf{x})$ therefore favors solutions with a rather low number of $N_K(\textbf{x})$ in an optimistic manner.

However, we want to penalize such solutions, as their sample means entail larger volatility and hence an increased risk of underestimating the true mean. For the present minimization problem (\ref{f1}), instead of subtracting $U_K(\textbf{x})$, we therefore propose the `pessimistic' counterpart to UCB by adding $U_K(\textbf{x})$.
Furthermore, one cannot assume $G(\textbf{x})$ to be bounded in DOvS problems. 
Nevertheless, to maintain the above idea and to balance $\overline{G}_K(\textbf{x})$ and $U_K(\textbf{x})$ independent of the underlying DOvS problem, we scale $\overline{G}_K(\textbf{x})$ to the interval $[0,1]$ by the following transformation 
$\tilde{G}_K(\textbf{x}) =(\overline{G}_K-B_K)/(W_K-B_K)$
where $B_K=\min_{\textbf{x}\in\mathcal{J}_K} \overline{G}_K(\textbf{x})$ and $W_K=\max_{\textbf{x}\in\mathcal{J}_K} \overline{G}_K(\textbf{x})$.
Similar transformations have been proposed in the literature, see \cite{pedroso2015tree} or \cite{neto2020multi}.
Consequently, we select the final solution $\hat{\textbf{x}}^*_K$ as follows:
\begin{equation}
\hat{\textbf{x}}^*_K=\argmin\limits_{\textbf{x} \in \mathcal{J}_K} \tilde{G}_K(\textbf{x})+U_K(\textbf{x}).
\label{f5}
\end{equation} 
Note that unlike UCB in the context of MABs, GMAB uses this criterion only to determine the solution considered to be best at the end of $K$ but not to identify the solution(s) to be visited in each iteration.
The latter are specified by $\mathcal{S}_k$ and hence by the genetic operators.
Keep in mind that alternatively using equation Equation \ref{f5} to determine the solutions to visit in each iteration would drastically increase the runtime, since this cannot be accomplished by the following proposed efficient memory scheme.
This runtime aspect is also one of the reasons why MABs usually focus on problems with a rather small number of solutions $\Theta$.

\subsection{Efficient Memory Control Scheme}
\label{sec:Efficient Memory Control Scheme}

As the number of iterations increases, so does the number of solutions in $\mathcal{V}_k$ and hence the memory of GMAB.
However, this also implies an increased computational effort when identifying the best $m$ solutions in $\mathcal{V}_{k-1}$, determining solutions that will be visited for the first time, and updating $\mathcal{V}_k$, cf. Lines 7, 9, and 13 of Algorithm \ref{algo:GMAB}.

To ensure the efficiency of these operations even for a large $\mathcal{V}_k$, the memory component of GMAB is built on two (balanced) binary search trees.
In the first tree, the so-called `lookup tree' (LUT), a new node is inserted whenever a solution $\textbf{x}$ is visited for the first time.
Within this tree, a node corresponds to an object with two attributes, namely, a so-called `solution code' and a `position code'.
The solution code serves as the search key of the corresponding node and decodes all elements $(x_1,\dots,x_D)^{\mathrm{T}}$ of a new solution $\textbf{x}$ by means of a positional numeral system to a unique integer:
$\sum_{d=1}^D  10^{ \lceil \mathrm{log}_{10}(x_d^{\mathrm{UB}}-x_d^{\mathrm{LB}}+1) \rceil  (D-d)}(x_d-x_d^{\mathrm{LB}}).$
The second attribute, the position code, is equal to the index of $\textbf{x}$ within the data structure where all visited solutions together with the corresponding $N(\textbf{x})$ and $R(\textbf{x})$ values are stored.
The LUT, implemented as an AVL tree, allows checking with a complexity of only $O(|\mathcal{V}_k|)$ if a solution already exists in that data structure or not. The same complexity holds true for inserting a new node \cite[]{adelsonvelskii1963algorithm}.

The second search tree, the so-called `sample average tree' (SAT), is a red--black tree whose $|\mathcal{V}_k|$ nodes each consist of an object with two attributes. 
For more details on red--black trees, see \cite{cormen2009introduction}.
The first attribute is the search key. For any solution $\textbf{x}\in\mathcal{V}_k$, it is equal to $\overline{G}_k(\textbf{x})$. Similar to the LUT, the position code of $\textbf{x}$ serves as the second attribute.
Note that the insert, delete, and search operations each have a complexity of $O(\mathrm{log}(|\mathcal{V}_k|))$ \cite[]{cormen2009introduction}.
The interaction of both trees enables time-efficient control of the memory, as described in the following.

In each iteration $k \geq 1$, the $m$ best solutions $\mathcal{E}_k$ are deleted from the SAT.
However, before deleting them, the corresponding position codes are stored in the cache $\mathcal{C}_k$.
This cache is at first required to (still) be able to identify $\mathcal{E}_k$ and later to reference the solutions that will be visited in iteration $k$, namely, $\mathcal{S}_k$.
In the next step, $\mathcal{E}_k$ is passed to the \textsc{GeneticModification}() function, which generates the offspring solutions $\mathcal{M}_k$.
For each $\textbf{x}\in\mathcal{M}_k$, the elements $(x_1,\dots,x_D)^{\mathrm{T}}$ are decoded to a unique number according to the positional numeral system. Based on this number, the LUT is used to check whether $\textbf{x}\in\mathcal{V}_{k-1}$ and whether there is already an LUT node associated with $\textbf{x}$. If $\textbf{x}\in\mathcal{V}_{k-1}$, the LUT returns the corresponding position code of $\textbf{x}$, which is added to $\mathcal{C}_k$. If $\textbf{x}$ is a new solution, it is appended to the data structure of all visited solutions, an associated node is inserted in the LUT, and the position code of $\textbf{x}$ is added to $\mathcal{C}_k$.
Next, the solutions to which $\mathcal{C}_k$ refers, i.e., $\mathcal{S}_k$, are visited, and hence, the respective $N_k(\textbf{x})$ and $R_k(\textbf{x})$ values are updated.
For each of those solutions, we insert a new node in the SAT based on the position code of $\textbf{x}$ and the updated sample average $\overline{G}_k(\textbf{x})=R_k(\textbf{x})/N_k(\textbf{x})$. 
Finally, at the end of iteration $k$, the cache is cleared.

Since all operations employed in the LUT and SAT have logarithmic complexity, the tree-based memory control scheme of GMAB is extremely time-efficient as the memory size $|\mathcal{V}_k|$ increases. It therefore facilitates a significant speed advantage compared to many other DOvS algorithms.

\section{Numerical Experiments}
\label{sec:Numerical Experiments}
Because the structure of $g(\textbf{x})$ is usually unknown, it is impossible to provide any theoretical performance in terms of the number of simulations required to achieve a certain solution quality.
Although, as discussed in Section \ref{sec:Stopping Criterion}, stopping criteria exist that provide a final solution $\hat{\textbf{x}}^*_K$ which guarantees a certain quality with some probability, it is completely unknown when this occurs.
Consequently, they have no impact on the performance.
To reasonably evaluate the latter (and to demonstrate the strengths of GMAB), extensive numerical experiments are required, which we provide in the following.
The main objectives of this section are to (i) highlight the remarkably good finite-budget performance of GMAB; (ii) investigate the effect of the parameters $p_\mathrm{cr}$, $p_\mathrm{mu}$, and $m$ on its performance; and (iii) examine the runtime of GMAB as its memory size, i.e., the number of visited solutions $|\mathcal{V}_k|$ increases.

We conducted a full factorial experiment with $p_\mathrm{cr}, p_\mathrm{mu}$ $\in \{0.05, 0.1,\dots,1\}$ and $m\in\{10,20,\dots,100\}$ for each test problem to examine the impact on the performance, see Section \ref{sec:Choice of Parameters}.
However, in practice, parameter tuning might be only possible to a limited extent due to cost-intensive simulations. To account for this aspect, all results for each test problem rely, unless otherwise stated, on the same parameterization, namely, $p_\mathrm{cr}=1$, $p_\mathrm{mu}=0.25$, and $m=20$. This parameterization achieved reasonably good and robust performance over all test problems. Nevertheless, note that for each problem, there are problem-specific superior parameterizations. To provide insights, we have included the results of the full factorial experiment in the supplementary material. Finally, note that all reported results are based on 400 independently conducted runs of GMAB to ensure robust results.

\subsection{Test Problems}
\label{sec:Test Problems}

For several reasons, only test problems from the DOvS literature were employed.
First, this provides a certain degree of impartiality, since the results are not biased by problems that may be specifically tailored to GMAB.
Second, it allows to some extent a comparison of the performance of GMAB with those of other DOvS algorithms reported in the literature.
For a detailed comparison, however, many more comprehensive numerical experiments are necessary, which are beyond the scope of this study.

We aimed to select a variety of heterogeneous problems.
For the first two of the following four problems, the exact surface of the objective function is unknown. They are thus representative of DOvS problems frequently encountered in practice.
We considered one problem with a comparatively small number of feasible solutions in a low- dimensional space and another one with a considerably larger number of dimensions and therefore a tremendous number of solutions.
The remaining two problems involve scenarios with known surfaces (and thus with known optimal solutions) where noise terms are added to provide stochasticity. In this regard, the focus was in particular on functions with multiple local optima, both in low and high-dimensional solution spaces, to investigate to what extent GMAB can escape from local optima.

The first problem, referred to as TP1, originates from \cite{koenig1985procedure} and was adopted by \cite{l2019gaussian} to analyze the performance of their Gaussian Markov improvement algorithm (GMIA).
The objective is to determine the optimal parameters of an $(s,S)$ inventory policy (with $x_1=s$ and $x_2=S-s$) that minimize the expected average cost per period over a planning horizon of 30 periods. The periodic demand is assumed to be a Poisson random variable with an associated mean of 25.
Similar to \cite{l2019gaussian}, we assume $1 \leq x_1,x_2 \leq 100$  and thus $|\Theta|=10\,000$.
Although the exact surface is unknown, we conducted Monte Carlo experiments and performed one million replications for each solution to visualize the objective function contour, see Figure \ref{TP1_1}(a).
Based on these experiments, the optimal inventory policy is at $x_1=17$ and $x_2=36$, causing an expected cost of $106.167$.

\begin{figure}[t]
\begin{center}
\includegraphics[height=2.2in]{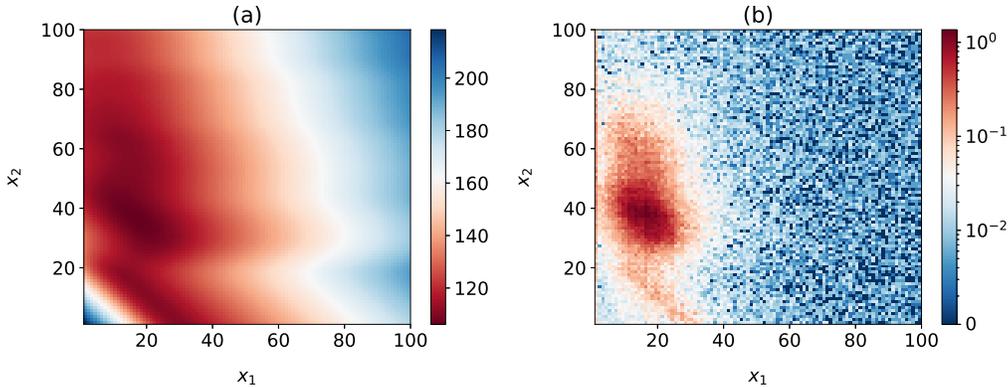}
\caption{(Color online) (a) Contour of the objective function in TP1 based on Monte Carlo experiments, (b) number of simulation replications of each solution conducted by GMAB averaged over all runs.} \label{TP1_1}
\end{center}
\end{figure}

The second problem, TP2, was first mentioned by \cite{hong2006discrete} to investigate the performance of their COMPASS approach and later applied by \cite{sun2014balancing} to compare it with the performance of their Gaussian process-based search (GPS). In TP2, the objective is to identify the optimal inventory capacities $x_1,\dots,x_8$ that maximize the expected total profit per period in an assemble-to-order system with eight items and five different types of customers 
(since (\ref{f1}) is assumed to be minimized by default, GMAB must be adapted to a maximization problem in TP2 accordingly).
Customer arrival is modeled through Poisson processes with customer-specific arrival rates.
Further stochasticity originates from normally distributed production times with item-specific expected values and standard deviations. Similar to the above authors, we assume $1\leq x_1,\dots,x_8 \leq 20$, resulting in a comparably large number of feasible solutions, namely, $|\Theta|=2.56 \times 10^{10}$. Unlike TP1 and due to this large number, we do not know the exact optimal solution with certainty. Additionally, the computational effort for obtaining replications is considerably larger compared to TP1.

The third problem, TP3, is the multimodal function proposed by \cite{sun2014balancing} in connection with their GPS approach. Unlike \cite{sun2014balancing}, however, we assume a minimization problem, which is why the sign of the objective function $g_3$ changes. Therefore, we have
\begin{equation}
g_3(x_1,x_2)=-10\bigg[\frac{\sin^6(0.05\pi \tfrac{1}{100}x_1)}{2^{2(\frac{x_1-90}{50})^2}}  +  \frac{\sin^6(0.05\pi \tfrac{1}{100}x_2)}{2^{2(\frac{x_2-90}{50})^2}}\bigg], \quad 0\leq x_1,x_2 \leq 10\,000,
\label{f7}
\end{equation} 
where a normally distributed noise term with zero mean and standard deviation of one is added to provide stochasticity. In TP3 with $|\Theta|\approx 10^{8}$, there are 25 local optima with a unique global optimum at $x_1^*=x_2^*=9\,000$ and an associated objective value of $g_3(x_1^*,x_2^*)=-20$.

The last problem, TP4, comprises the following high-dimensional function to be minimized with $2^D$ local optima
\begin{equation}
g_4(x_1,\dots,x_D)=-\sum_{d=1}^D(\beta_1\mathrm{exp}\{-\gamma_1(x_d-\xi_1^*)^2\} +  \beta_2\mathrm{exp}\{-\gamma_2(x_d-\xi_2^*)^2\}), \quad -100 \leq x_d \leq 100,
\label{f8}
\end{equation} 
where $\beta_1=300$, $\beta_2=500$, $\gamma_1=0.001$, $\gamma_2=0.005$, $\xi_1^*=-38$, and $\xi_2^*=56$. 
As in TP3, a noise term with zero mean and standard deviation of one is added.
As $|\Theta|=201^{D}$ increases exponentially with the number of dimensions, TP4 involves a tremendous number of feasible solutions.
Each local optima is of the form $x_d^*\in\{\xi_1^*, \xi_2^*\}$ $\forall d=1,\dots,D$ with a unique global optimum at $x_d^*=\xi_2^*$ $\forall d=1,\dots,D$ and a corresponding objective function value of $\approx -500.04\, D.$
The problem was first studied by \cite{xu2013adaptive} in the context of the comparison between IS-AHA and IS-COMPASS, two modified versions of AHA and COMPASS.
Going beyond \cite{xu2013adaptive}, we also consider a noise term with standard deviation of $100D$ in a second experiment to investigate the performance of GMAB with regard to a high-dimensional multimodal function with large(r) noise.

\subsection{Empirical Results}
\label{sec:Empirical Results}

For comparison purposes, 
we consider the performance plots over the course of a fixed number of simulation replications.
This means that the stopping criterion in Line 6 of Algorithm \ref{algo:GMAB} becomes this fixed number.
Hence, the solution considered to be best is captured after certain steps of simulation replications (and not after certain steps of iterations).
To avoid any confusion, we omit the iteration index with regard to $\hat{\textbf{x}}^*$, as long as a fixed number of simulation replications is used as the stopping criterion.

Whenever the exact value of the optimal solution is known (i.e., in case of TP3 and TP4), the performance is evaluated based on this true value. While we use the simulation results of separate Monte Carlo experiments for TP1, the performance in case of TP2 relies on the estimated values.

Before further analyzing the performance with respect to the objective function values of TP1, we illustrate the average number of simulation replications of each solution over 400 GMAB runs given a total budget of 500 replications in Figure \ref{TP1_1}b.
To better highlight the differences, a log base 10 scale was chosen.
After a total of only 500 simulation replications, it is already evident that GMAB visits the global optimum and its neighborhood on average considerably more often than non-promising regions of the solution space, which indicates a fast convergence.
In the following, we contrast the performance of GMAB with that of GMIA. The latter algorithm by \cite{l2019gaussian} was also applied to TP1 in their study (in connection with their proposed `complete expected improvement' criterion).
For comparison purposes, we examine the optimality gap (based on the results of the separately conducted Monte Carlo experiments) over the course of $2\,000$ simulation replications.
Each trajectory in Figure \ref{TP1_2} is averaged over 400 runs and starts after the respective initialization phase, which is why that of GMIA originates at a higher number of simulation replications. For detailed information about the setting of GMIA, we refer to \cite{l2019gaussian}.
From the trajectories in Figure \ref{TP1_2}, it is apparent that GMAB achieves a fast convergence and requires fewer simulation replications than GMIA to generate a smaller optimality gap at the same time.

\begin{figure}[t]
\begin{center}
\includegraphics[height=2.2in]{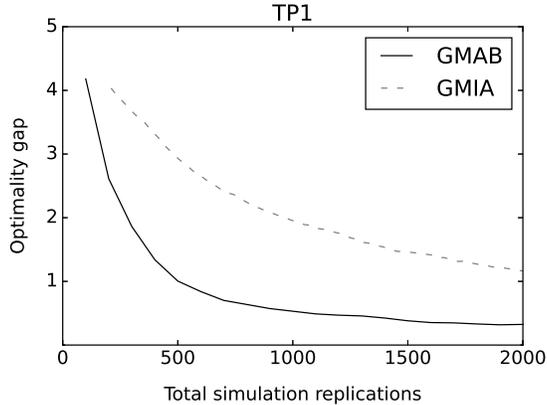}
\caption{Trajectories of the optimality gap in TP1 for GMAB and GMIA averaged over all runs.} \label{TP1_2}
\end{center}
\end{figure}

For TP2, we assume a total budget of $15\,000$ simulation replications in each run. Figure \ref{TP2_1} plots the performance of GMAB against the reported performances of COMPASS and GPS. For more information about the settings of the latter two applied in TP2, we refer to \cite{hong2006discrete} and \cite{sun2014balancing}.
As in TP1, GMAB performs considerably better than the benchmark algorithms and converges very quickly even in this higher-dimensional problem.

\begin{figure}
\begin{center}
\includegraphics[height=2.2in]{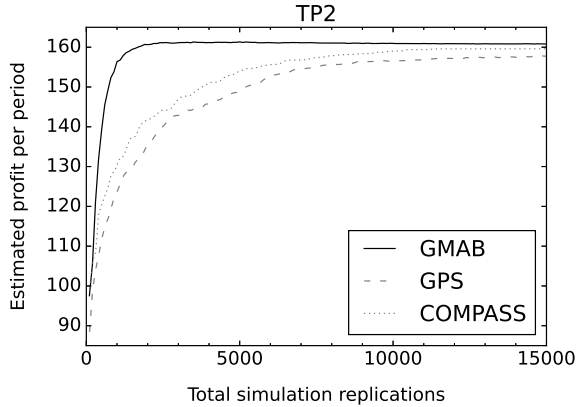}
\caption{Estimated profits per period in TP2 for GMAB, GPS, and COMPASS averaged over all runs.} \label{TP2_1}
\end{center}
\end{figure}

In addition, Table \ref{tab:TP2_Table} provides insight into the impact of three different final selection criteria (FSC) on GMAB's performance. FSC1 corresponds to the default criterion proposed in Section \ref{sec:FSC}, which is consequently used for all test problems and thus applied to GMAB in Figure \ref{TP2_1}. According to FSC2 and FSC3, the solution with the highest estimated value and the solution with the largest number of simulation replications are selected, respectively. As the sample means $\overline{G}(\hat{\textbf{x}}^*)$ are typically biased, we also report the `true' values after certain replication steps.
In this case, `true' means the result of $10\,000$ simulation replications of $\hat{\textbf{x}}^*$ performed independently and separately through a Monte Carlo experiment.
In addition, $N(\hat{\textbf{x}}^*)$ provides information on how many times the solution considered best was selected after certain simulation replication steps. Note that the provided $\overline{G}(\hat{\textbf{x}}^*)$, `true' as well as $N(\hat{\textbf{x}}^*)$ values are averaged over 400 runs.
As shown in Table \ref{tab:TP2_Table}, FSC1 achieves the best (true) performance.
FSC2 is overly optimistic by always selecting the solution with the current best sample mean. Most of the time, this solution is associated with low simulation replications, so its sample mean tends to rely on outliers with considerably worse true values.
Even though FSC3 performs comparatively well, the criterion of always selecting $\hat{\textbf{x}}^*$ as the solution with the most simulation replications turns out to be too conservative.

While the overestimation of the true values is obvious for FSC2, it also slightly arises in case of FSC1 and FSC3 due to the low $N(\hat{\textbf{x}}^*)$ values implied by the early stage of the search process.
However, the gap decreases over the course of total simulation replications.
The `true' performance therefore improves and does not decline, which could be inferred incorrectly based on $\overline{G}(\hat{\textbf{x}}^*)$ with respect to FSC1 in Table \ref{tab:TP2_Table}.

\begin{table*}[hbt]
\setlength{\tabcolsep}{3pt}
   \centering
   \caption{Comparison of the impact of different final selection criteria on the performance.}
   \label{tab:TP2_Table}
   \renewcommand{\arraystretch}{1.4}
   \begin{tabular}{c@{\hspace{1.2em}}cc@{\quad}ccc@{\hspace{1em}}ccc@{\hspace{1em}}ccc}
     \toprule
		&\multicolumn3c{FSC1}&&\multicolumn3c{FSC2}&&\multicolumn3c{FSC3}\\
		\cmidrule{2-4}\cmidrule{6-8} \cmidrule{10-12}
		\raisebox{0pt}[0pt][0pt]{\begin{tabular}[b]{@{}c@{}}Simulation\\[-.5mm]replications\end{tabular}} &  $\overline{G}(\hat{\textbf{x}}^*)$ & `true' & $N(\hat{\textbf{x}}^*)$ &&    $\overline{G}(\hat{\textbf{x}}^*)$ &  `true' & $N(\hat{\textbf{x}}^*)$& &  $\overline{G}(\hat{\textbf{x}}^*)$ &  `true' & $N(\hat{\textbf{x}}^*)$\\
     \midrule
		\phantom{1}1\,000  &  156.36 & 151.98 &   8.95 && 165.13 & 151.83 &  1.59 && 150.40 & 147.87 &  12.20\\
		\phantom{1}2\,000  &  160.64 & 157.37 &  20.65 && 168.74 & 155.45 &  2.04 && 157.81 & 155.83 &  26.16\\
		\phantom{1}3\,000  &  161.13 & 158.53 &  34.69 && 168.93 & 155.83 &  2.24 && 159.21 & 157.62 &  44.35\\
		\phantom{1}4\,000  &  161.17 & 159.13 &  50.86 && 168.93 & 155.87 &  2.38 && 159.61 & 158.41 &  63.99\\
		\phantom{1}5\,000  &  161.34 & 159.45 &  64.30 && 168.70 & 155.65 &  2.19 && 159.81 & 158.80 &  84.58\\
		10\,000            &  160.94 & 160.03 & 164.33 && 166.97 & 155.18 &  5.63 && 159.98 & 159.50 & 202.88\\
		15\,000            &  160.80 & 160.21 & 274.17 && 166.15 & 154.82 & 14.19 && 159.99 & 159.70 & 330.29\\
		 \bottomrule
\end{tabular}
\end{table*}

Since the optimal solution of TP3 is known, we evaluate the performance based on the optimality gap.
Also in this problem, GMAB performs noticeably better than the benchmark from the literature, cf. Figure \ref{TP3}. 
In the latter, the trajectories of GMAB as well as those of the benchmark algorithm GPS are averaged over all runs. For more details about the setup of GPS in TP3, we refer to \cite{sun2014balancing}.

Furthermore, we examine the performance across the 400 different runs of GMAB.
The light gray area indicates the region in which the trajectories of all GMAB runs are located.
The upper and the lower edges of this area represent the worst and best values, respectively,
that were scored within the different runs at certain simulation replications. Hence, the edges do not correspond to trajectories of particular runs but to the 0th and 100th percentiles.
The edges of the slightly darker gray area are the 5th and 95th percentiles, with the edges of the innermost darkest gray area being the 25th and 75th percentiles.
In each of the 400 runs, GMAB identifies the neighborhood surrounding the global optimum after at most only $3\,000$ total simulation replications (the second-best (local) optima are associated with an optimality gap of 2).
Accordingly, the results demonstrate GMAB's promising performance also in the low-dimensional but multimodal TP3 and hence the ability to escape from several local optima. 

\begin{figure}
\begin{center}
\includegraphics[height=2.2in]{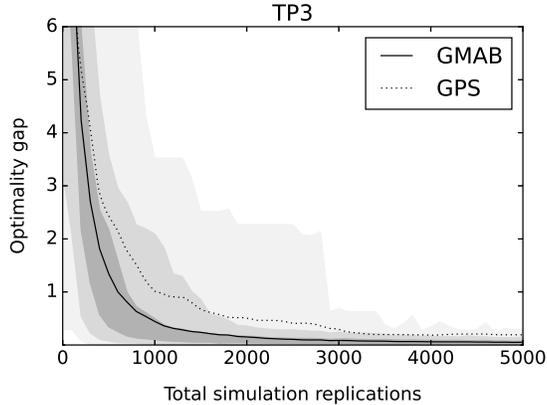}
\caption{Trajectories of the optimality gap in TP3 for GMAB and GPS averaged over all runs.} \label{TP3}
\end{center}
\end{figure}

We now move to the investigation of how GMAB performs when optimizing high-dimensional multimodal functions.
For this purpose, four instances of TP4 with different dimensions $D=5,10,15,20$ are considered, which are referred to as TP4\_D05, TP4\_D10, TP4\_D15 and TP4\_D20. As mentioned previously, $2^D$ local optima exist in each problem. In addition to the default setting with $m=20$, we also report GMAB's performance when $m=50$ and $m=100$. Figure \ref{TP4} shows the trajectories for the different settings averaged over 400 runs for each problem. Furthermore, Table \ref{tab:TP4_Table} provides detailed information about $g(\hat{\textbf{x}}^*)$ and the optimality gap denoted by $\Delta(\hat{\textbf{x}}^*)$ after a total of $10^5$ simulation replications.
In addition to the problems from the literature that assume a noise term with standard deviation of one, see \cite{xu2013adaptive}, Table \ref{tab:TP4_Table} reports GMAB's performance in all instances of TP4 considering a considerably higher standard deviation of $100D$. Obviously, a larger noise term is more challenging, as it raises the difficulty of the problem. In practice, one would expect to choose a larger budget for such problems.

As can be seen from Figure \ref{TP4}, the optimality gap decreases more quickly in case of $m=20$.
However, this advantage subsides after a certain number of replications and better performance is achieved with larger $m$ values.
This result is not surprising, since $m$ specifies the number of solutions provided for the genetic operators in each iteration. 
While at early stages of GMAB the heterogeneity of solutions within $\mathcal{E}_k$ is largest, it decreases over the course of iterations due to the genetic operators.
At the start of the search, the crossover operator causes large jumps within the solution space by recombining solutions and therefore quickly `explores' promising areas. 
Simultaneously, the mutation operator usually `exploits' similarly good or better solutions mostly located near promising ones. 
As a consequence, the solutions in $\mathcal{E}_k$ become increasingly similar over the course of time, which decreases exploration (cf. the supplementary material of this paper for an investigation of the heterogeneity of solutions within $\mathcal{E}_k$).
However, the larger $m$, the slower this process proceeds.  
Consequently, a large $m$ value causes a longer lasting exploration phase. This is of central importance for the present high-dimensional problems involving multiple local optima.
Therefore, $m=100$ achieves the best performance in all four instances of TP4.
Even though low $m$ values are detrimental in such problems, global convergence (cf. Section \ref{sec:Convergence of GMAB}) is ensured irrespective of the choice of $m$.
In the following section, we discuss in more detail how the performance of GMAB is affected by $m$ and also the probabilities $p_\mathrm{cr}$, $p_\mathrm{mu}$.

To compare the performance of GMAB with those of benchmark algorithms from the literature, we refer to Table 3 in \cite{xu2013adaptive} for detailed information on the performance of the benchmark algorithms IS-AHA and IS-COMPASS. To briefly summarize, GMAB generates superior results even with the suboptimal parametrization of $m=20$ in all four instances of TP4.
For example in case of TP4\_D20, while IS-AHA and IS-COMPASS achieve an optimality gap of 22.4\% and 20\%, respectively, GMAB's optimality gap are 13\%, 5.9\%, and 3\% when using $m=20$, $m=50$, and $m=100$, respectively. Furthermore, note that to achieve these results, GMAB requires even fewer simulation replications than both benchmark algorithms.

\begin{figure}
\begin{center}
\includegraphics[height=4.55in]{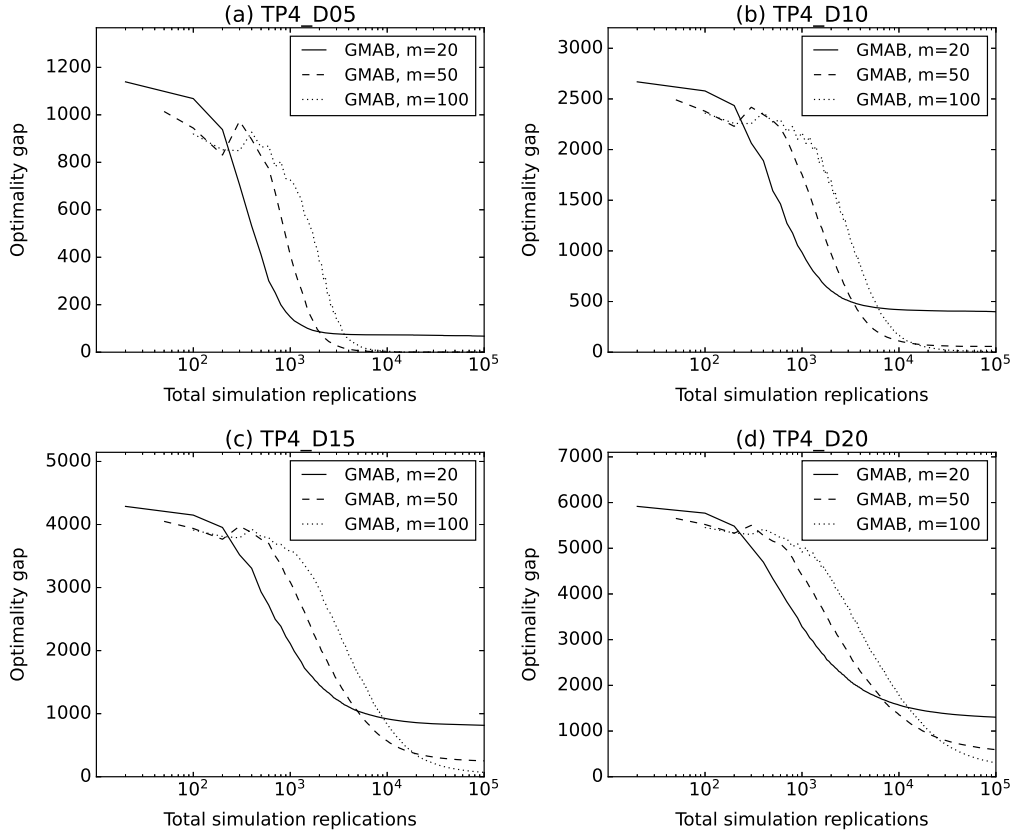}
\caption{Trajectories of GMAB's optimality gap averaged over all runs in TP4\_D05, TP4\_D10, TP4\_D15, and TP4\_D20 assuming different parametrizations of $m$.} \label{TP4}
\end{center}
\end{figure}

\begin{table*}
\setlength{\tabcolsep}{3pt}
   \centering
   \caption{Optimality gap after $10^5$ simulation replications averaged over all runs.}
   \label{tab:TP4_Table}
   \renewcommand{\arraystretch}{1.4}
   \begin{tabular}{c@{\hspace{1.2em}}cc@{\hspace{3.2em}}cccccccc}
     \toprule
		&&&\multicolumn2c{$m=20$}&&\multicolumn2c{$m=50$}&&\multicolumn2c{$m=100$}\\
		\cmidrule{4-5}\cmidrule{7-8} \cmidrule{10-11}
		&\raisebox{0pt}[0pt][0pt]{\begin{tabular}[b]{@{}c@{}}\,\\[-.5mm]Problem\end{tabular}} & $g(\textbf{x}^*)$ & $g(\hat{\textbf{x}}^*)$ & $\Delta(\hat{\textbf{x}}^*)$  &&    $g(\hat{\textbf{x}}^*)$ &  $\Delta(\hat{\textbf{x}}^*)$ & &  $g(\hat{\textbf{x}}^*)$ &  $\Delta(\hat{\textbf{x}}^*)$ \\
     \midrule
		\multirow{4}{*}{\STAB{\rotatebox[origin=c]{90}{std. dev. = $1$}}} 
		&TP4\_D05  & $-2500.22$ &  $-2432.66$  &   \phantom{11}67.56  && $-2499.71$  & \phantom{11}0.51 && $-2500.22$ & \phantom{11}0.00 \\
		&TP4\_D10  & $-5000.44$&  $-4601.16$  &   \phantom{1}399.28  && $-4943.41$  & \phantom{1}57.03 && $-4993.42$ & \phantom{11}7.02 \\
		&TP4\_D15  & $-7500.65$&  $-6685.09$  &  \phantom{1}815.56  &&  $-7248.78$  & 251.87 && $-7431.43$ & \phantom{1}69.22 \\
		&TP4\_D20  & $-10000.87$&  $-8698.10$  &  $1302.77$  && $-9408.10$ & $592.77$  && $-9696.93$ & $303.94$ \\
		\midrule
		\multirow{4}{*}{\STAB{\rotatebox[origin=c]{90}{std. dev. = $100D$}}} 
		&TP4\_D05  & $-2500.22$&  $-2396.85$  &  \phantom{1}103.37  && $-2473.14$ &  \phantom{11}27.08  && $-2479.60$ &  \phantom{11}20.62  \\
		&TP4\_D10  & $-5000.44$&  $-4440.44$  &  \phantom{1}560.00  && $-4749.65$ &  \phantom{1}250.79 && $-4869.02$ &  \phantom{1}131.42 \\
		&TP4\_D15  & $-7500.65$&  $ -6253.83$  & 1246.82&& $-6678.83$ &  \phantom{1}821.82 && $-6938.78$ &  \phantom{1}561.87\\
		&TP4\_D20  & $-10000.87$&  $-7983.44$  & 2017.43&& $-8333.43$ & $1667.44$&& $-8557.02$ & $1443.85$\vspace{10pt}\\		
  \bottomrule
\end{tabular}
\end{table*}

\subsection{Choice of Parameters}
\label{sec:Choice of Parameters}

Although the parametrization $p_\mathrm{cr}=1$, $p_\mathrm{mu}=0.25$, and $m=20$ was used in all the above test problems, it is not a one-size-fits-all configuration. 
A unique setting for all problems was chosen primarily to reflect the fact that in practical problems, parameter tuning is often possible only to a limited extent.
In general, the best choice of parameters depends strongly on the problem at hand. Nevertheless, in the following, we will give some general recommendations for the choice of parameters and discuss how they affect the performance of GMAB.

As mentioned in Section \ref{sec:Empirical Results}, the crossover operator mainly supports the exploration of new solutions at the start and shifts to exploitation as the iterations progress, which is generally a desirable behavior.
In our experiments, a high crossover probability therefore tended to perform better (in all problems but especially in high-dimensional multimodal problems).  
We did not observe the opposite effect in any problem (neither in problems from Section \ref{sec:Test Problems} nor in problems we studied but do not discuss here). 
For this reason, we recommend choosing a large crossover probability, such as $p_\mathrm{cr}=1$, in general.

For $p_\mathrm{mu}$ and $m$, we found a strong dependence on the underlying problem.
Furthermore, both parameters affect each other, which is why they should not be determined independently.
In multimodal problems where the local optima are more distant from each other, which applies, e.g., to all instances of TP4, larger $m$ values tended to perform better.
As already mentioned, this is because the resulting longer-lasting heterogeneity of solutions in $\mathcal{E}_k$ causes a longer exploration phase, which helps to identify global optima.
In combination with a large $m$, a low mutation probability $p_\mathrm{mu}$ was often advantageous, as this simultaneously reinforces exploitation. 
However, a large $m$ is not always be recommended. In TP3, for example, the local optima are very close to each other. In such problems a larger $m$ is disadvantageous, because simulation resources are overspent for exploration and hence, are not available for exploitation.
In particular, the latter is essential to identify actual global optima, especially in regions with solutions of similar quality and high noise, such as in TP1 and TP2. Here, a smaller $m$ is preferable.
For demonstration purposes, Figure \ref{PAPER_Parameters} visualizes GMAB's performance at fixed $p_\mathrm{cr}=1$ and variable $p_\mathrm{mu}$ and $m$ for TP1 and TP4\_D05.
The choice of parameters, apart from the problem itself, also depends on the available budget, as seen in Figure \ref{TP4}.

Consequently, parametrization is a difficult task, especially given that usually no information about the problem is available.
Regardless, GMAB achieves excellent performance even with suboptimal parametrizations, as for each test problem, the majority of the different configurations tested generated better results than those of the benchmark algorithms from the literature. 
To give an example, the reported optimality gaps for IS-COMPASS and IS-AHA are approximately 300 and 400, respectively, for TP\_D05 \cite[]{xu2013adaptive}. 
With a similar number of replications as in the latter study, GMAB achieves better results for almost all parameter combinations, cf. Figure \ref{PAPER_Parameters}b.
The same applies to the performance of GMAB in TP1 compared to that of GMIA, cf. Figure \ref{PAPER_Parameters}a and Figure \ref{TP1_2}.

\begin{figure}[t]
\begin{center}
\includegraphics[height=2.2in]{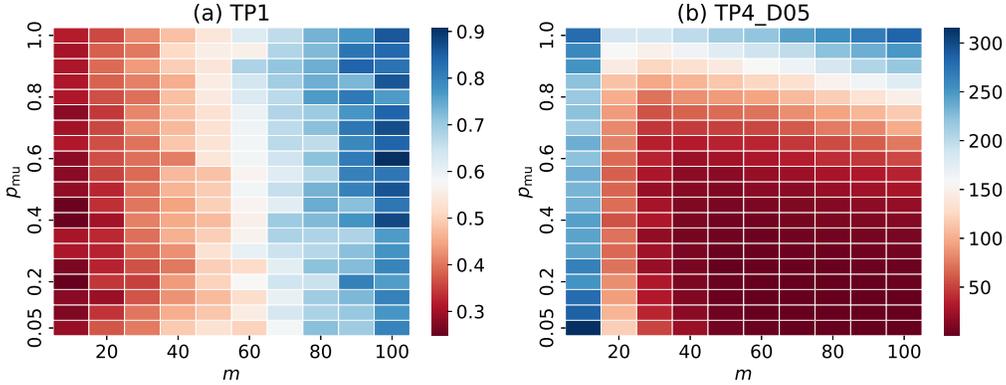}
\caption{(Color online) Optimality gap averaged over all runs with fixed $p_\mathrm{cr}=1$ and variable $p_\mathrm{mu}$ and $m$ after 2\,000 and 10\,000 simulation replications in TP1 and TP4\_D05, respectively.} \label{PAPER_Parameters}
\end{center}
\end{figure}

\subsection{Runtime}
\label{sec:Runtime}

In addition to the performance in terms of objective function values, the required runtime is a second important evaluation criterion.
Similar to \cite{sun2014balancing}, we distinguish two different problem categories: for problems of the first category, sampling simulation replications is significantly more time-consuming than the overhead of the actual algorithm.
The comparatively cost-intensive simulations in TP2 serve as an example.
For problems of the second category, the computational effort of taking simulation replications is negligible compared to the overhead of the algorithm itself. In this regard, TP3 and TP4 serve as examples.

Since GMAB requires fewer simulation replications to achieve similarly good or superior results than the benchmarks from the literature, see, e.g., Figure \ref{TP2_1}, its runtime for problems of the first category is, as a result, considerably lower.
However, the runtime advantage becomes much more pronounced in problems of the second category.
Since all operations related to memory control have logarithmic complexity, cf. Section \ref{sec:Efficient Memory Control Scheme}, the runtime per iteration barely rises as the number of visited solutions increases.
For illustration, Figure \ref{RUNTIME} shows the runtime of GMAB per iteration as a function of the memory size $|\mathcal{V}_k|$ in TP4 with $D=5,10,15,20$ averaged over all 400 runs.
Again, the default parametrization $p_\mathrm{cr}=1$, $p_\mathrm{mu}=0.25$, and $m=20$ was used.
As can be seen, the size of $\mathcal{V}_k$ hardly affects the runtime per iteration.
By contrast, the computational effort increases with the number of dimensions. This is obvious since
the larger $D$ the more possibilities exist for crossover and the more likely mutations will occur.
The actual computational effort for GMAB therefore occurs almost only with the genetic operators.

Furthermore, Table \ref{tab:Runtime} reports the runtimes of GMAB with respect to all test problems from Section \ref{sec:Test Problems} given a budget of $10^5$ simulation replications. In addition, we provide the average number of iterations executed within this budget.
The lower $D$ and the smaller $|\Theta|$, the more likely genetic operators generate offspring solutions $\textbf{x}\in\mathcal{M}_k$ that are already contained in $\mathcal{E}_k$. Since only solutions $\textbf{x}\in\mathcal{S}_k=\mathcal{E}_k\cup\mathcal{M}_k$ are visited in iteration $k$, a large intersection of  $\mathcal{E}_k$ and $\mathcal{M}_k$ will reduce $|\mathcal{S}_k|$.
For this reason, given a fixed budget of simulation replications, more iterations are executed on average in case of TP1 than, for example, in case of TP4\_D20.
With regard to the total runtimes, TP2 suffers from (GMAB independent) high costs due to performing simulation replications. All other test problems reveal the extremely low runtime of GMAB.
Compared to the benchmark algorithms, which have been applied to the same problems in previous studies, GMAB is considerably faster. 
To provide an example, GMAB requires approximately 0.02 seconds to achieve the reported results in Figure \ref{TP3}, while GPS, in contrast, requires approximately 10 seconds for a similar performance.
For more information regarding the runtimes of the benchmark algorithms, see \cite{xu2013adaptive,sun2014balancing,l2019gaussian,semelhago2021rapid}.
An in-depth runtime comparison is, however, beyond the scope of this study and turns out to be difficult, since the implementation language and the hardware used affect the runtime as well. Regardless, the main reason for GMAB's runtime advantage is primarily due to the logarithmic complexity of its memory-related operators.

Finally, to further improve the runtime of GMAB (especially for problems with high simulation costs, such as TP2), the simulation of all $\textbf{x}\in\mathcal{S}_k$ could be executed in parallel on multiple cores. 
Such an implementation would offer more (runtime) capacities for tasks like problem-specific parameter tuning.
However, we refrained from parallel computing in our computational experiments.

\begin{figure}
\begin{center}
\includegraphics[height=2.2in]{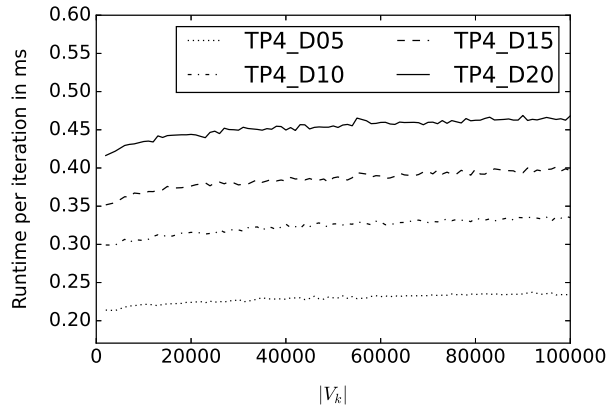}
\caption{Required runtimes per iteration in milliseconds (ms) for TP4\_D05, TP4\_D10, TP4\_D15, and TP4\_D20 averaged over all runs.} \label{RUNTIME}
\end{center}
\end{figure}


\begin{table*}[hbt]
\caption{Required runtimes in seconds (s) for all test problems given a budget of $10^5$ simulation replications and the number of iterations executed within this budget averaged over all runs.}
\setlength{\tabcolsep}{3pt}
  \centering
  \label{tab:Runtime}
	\begin{tabular}{l@{\hspace{1.6em}}r@{\hspace{1.6em}}r@{\hspace{1.6em}}r@{\hspace{1.6em}}r@{\hspace{1.6em}}r@{\hspace{1.6em}}r@{\hspace{1.6em}}r}
	\toprule
	& TP1 & TP2 & TP3 & TP4\_D05 & TP4\_D10 & TP4\_D15 & TP4\_D20 \\
	\cmidrule{2-8}
	Runtime (s) & 1.79 & 315.83 & 0.45 & 0.56 & 0.76 & 0.90 & 1.06 \\
	Iterations & 3218.66 & 2686.95 & 2818.20 & 2780.91 & 2565.60 & 2519.70 & 2518.37 \\
	\bottomrule
\end{tabular}
\end{table*}

\section{Conclusion}
\label{sec:Conclusion}

This paper proposes a new algorithm that combines concepts from the reinforcement learning domain of multi-armed bandits and random search components from the domain of genetic algorithms to solve discrete stochastic optimization problems via simulation.
We show that the so-called GMAB algorithm converges with probability 1 to the set of globally optimal solutions as the simulation effort increases.
The empirical experiments demonstrate that GMAB achieves considerably superior results at lower runtime compared to related benchmark algorithms from the literature in all considered test problems. They further emphasize the excellent performance of GMAB even in high-dimensional multimodal problems with different levels of noise and an enormous number of feasible solutions. 
Based on these results, we are convinced that GMAB should become a core dimension of stochastic optimization.

The proposed approach offers many opportunities for further extensions and investigations.
First, the impact of several other forms of operators offered by the genetic algorithm literature 
or additional techniques like adaptive mutation step size control
should also be analyzed in the context of GMAB.
Second, and most importantly, the idea of embedding random search strategies is not restricted to strategies borrowed from genetic algorithms. We have already made initial attempts to generalize it to a framework where other metaheuristics are embedded as well in order to solve discrete stochastic optimization problems.
Third, simulation experiments are usually associated with large computational effort in practical applications. Against this background, the already mentioned concept of running simulation replications in parallel, should be further elaborated within the latter framework.

\bibliographystyle{model5-names}
\bibliography{references}  

\begin{thebibliography}{32}
\expandafter\ifx\csname natexlab\endcsname\relax\def\natexlab#1{#1}\fi
\providecommand{\url}[1]{\texttt{#1}}
\providecommand{\href}[2]{#2}
\providecommand{\path}[1]{#1}
\providecommand{\DOIprefix}{doi:}
\providecommand{\ArXivprefix}{arXiv:}
\providecommand{\URLprefix}{URL: }
\providecommand{\Pubmedprefix}{pmid:}
\providecommand{\doi}[1]{\href{http://dx.doi.org/#1}{\path{#1}}}
\providecommand{\Pubmed}[1]{\href{pmid:#1}{\path{#1}}}
\providecommand{\bibinfo}[2]{#2}
\ifx\xfnm\relax \def\xfnm[#1]{\unskip,\space#1}\fi
\bibitem[{Adelson-Velsky \& Landis(1963)}]{adelsonvelskii1963algorithm}
\bibinfo{author}{Adelson-Velsky, M.}, \& \bibinfo{author}{Landis, E.~M.}
  (\bibinfo{year}{1963}).
\newblock {\it \bibinfo{title}{An algorithm for the organization of
  information}\/}.
\newblock \bibinfo{type}{Technical Report} Joint Publications Research Service,
  Washington DC.
\bibitem[{Alrefaei \& Andrad{\'o}ttir(1999)}]{alrefaei1999simulated}
\bibinfo{author}{Alrefaei, M.~H.}, \& \bibinfo{author}{Andrad{\'o}ttir, S.}
  (\bibinfo{year}{1999}).
\newblock \bibinfo{title}{A simulated annealing algorithm with constant
  temperature for discrete stochastic optimization}.
\newblock {\it \bibinfo{journal}{Management Science}\/},  {\it
  \bibinfo{volume}{45}\/}, \bibinfo{pages}{748--764}.
  \DOIprefix\doi{https://doi.org/10.1287/mnsc.45.5.748}.
\bibitem[{Amaran et~al.(2016)Amaran, Sahinidis, Sharda \&
  Bury}]{amaran2016simulation}
\bibinfo{author}{Amaran, S.}, \bibinfo{author}{Sahinidis, N.~V.},
  \bibinfo{author}{Sharda, B.}, \& \bibinfo{author}{Bury, S.~J.}
  (\bibinfo{year}{2016}).
\newblock \bibinfo{title}{Simulation optimization: A review of algorithms and
  applications}.
\newblock {\it \bibinfo{journal}{Annals of Operations Research}\/},  {\it
  \bibinfo{volume}{240}\/}, \bibinfo{pages}{351--380}.
  \DOIprefix\doi{https://doi.org/10.1007/s10479-015-2019-x}.
\bibitem[{Andrad{\'o}ttir(1995)}]{andradottir1995method}
\bibinfo{author}{Andrad{\'o}ttir, S.} (\bibinfo{year}{1995}).
\newblock \bibinfo{title}{A method for discrete stochastic optimization}.
\newblock {\it \bibinfo{journal}{Management Science}\/},  {\it
  \bibinfo{volume}{41}\/}, \bibinfo{pages}{1946--1961}.
  \DOIprefix\doi{https://doi.org/10.1287/mnsc.41.12.1946}.
\bibitem[{Andrad{\'o}ttir(1996)}]{andradottir1996global}
\bibinfo{author}{Andrad{\'o}ttir, S.} (\bibinfo{year}{1996}).
\newblock \bibinfo{title}{A global search method for discrete stochastic
  optimization}.
\newblock {\it \bibinfo{journal}{SIAM Journal on Optimization}\/},  {\it
  \bibinfo{volume}{6}\/}, \bibinfo{pages}{513--530}.
  \DOIprefix\doi{https://doi.org/10.1137/0806027}.
\bibitem[{Andrad{\'o}ttir \& Prudius(2009)}]{andradottir2009balanced}
\bibinfo{author}{Andrad{\'o}ttir, S.}, \& \bibinfo{author}{Prudius, A.~A.}
  (\bibinfo{year}{2009}).
\newblock \bibinfo{title}{Balanced explorative and exploitative search with
  estimation for simulation optimization}.
\newblock {\it \bibinfo{journal}{INFORMS Journal on Computing}\/},  {\it
  \bibinfo{volume}{21}\/}, \bibinfo{pages}{193--208}.
  \DOIprefix\doi{https://doi.org/10.1287/ijoc.1080.0309}.
\bibitem[{Auer et~al.(2002)Auer, Cesa-Bianchi \& Fischer}]{auer2002finite}
\bibinfo{author}{Auer, P.}, \bibinfo{author}{Cesa-Bianchi, N.}, \&
  \bibinfo{author}{Fischer, P.} (\bibinfo{year}{2002}).
\newblock \bibinfo{title}{Finite-time analysis of the multiarmed bandit
  problem}.
\newblock {\it \bibinfo{journal}{Machine Learning}\/},  {\it
  \bibinfo{volume}{47}\/}, \bibinfo{pages}{235--256}.
  \DOIprefix\doi{https://doi.org/10.1023/A:1013689704352}.
\bibitem[{Bubeck \& Cesa-Bianchi(2012)}]{bubeck2012regret}
\bibinfo{author}{Bubeck, S.}, \& \bibinfo{author}{Cesa-Bianchi, N.}
  (\bibinfo{year}{2012}).
\newblock \bibinfo{title}{Regret analysis of stochastic and nonstochastic
  multi-armed bandit problems}.
\newblock {\it \bibinfo{journal}{Foundations and Trends in Machine
  Learning}\/},  {\it \bibinfo{volume}{5}\/}, \bibinfo{pages}{1--122}.
  \DOIprefix\doi{https://doi.org/10.1561/2200000024}.
\bibitem[{Cormen et~al.(2009)Cormen, Leiserson, Rivest \&
  Stein}]{cormen2009introduction}
\bibinfo{author}{Cormen, T.~H.}, \bibinfo{author}{Leiserson, C.~E.},
  \bibinfo{author}{Rivest, R.~L.}, \& \bibinfo{author}{Stein, C.}
  (\bibinfo{year}{2009}).
\newblock {\it \bibinfo{title}{Introduction to Algorithms}\/}.
\newblock (\bibinfo{edition}{3rd} ed.).
\newblock \bibinfo{address}{Cambridge, MA}: \bibinfo{publisher}{MIT Press}.
\bibitem[{Hinterding(1995)}]{hinterding1995gaussian}
\bibinfo{author}{Hinterding, R.} (\bibinfo{year}{1995}).
\newblock \bibinfo{title}{Gaussian mutation and self-adaption for numeric
  genetic algorithms}.
\newblock In {\it \bibinfo{booktitle}{Proceedings of 1995 IEEE International
  Conference on Evolutionary Computation}\/} (p. \bibinfo{pages}{384}).
\newblock \bibinfo{organization}{IEEE} volume~\bibinfo{volume}{1}.
\newblock \DOIprefix\doi{https://doi.org/10.1109/ICEC.1995.489178}.
\bibitem[{Hong et~al.(2021)Hong, Fan \& Luo}]{hong2021review}
\bibinfo{author}{Hong, L.~J.}, \bibinfo{author}{Fan, W.}, \&
  \bibinfo{author}{Luo, J.} (\bibinfo{year}{2021}).
\newblock \bibinfo{title}{Review on ranking and selection: A new perspective}.
\newblock {\it \bibinfo{journal}{Frontiers of Engineering Management}\/},  {\it
  \bibinfo{volume}{8}\/}, \bibinfo{pages}{321--343}.
  \DOIprefix\doi{https://doi.org/10.1007/s42524-021-0152-6}.
\bibitem[{Hong \& Nelson(2006)}]{hong2006discrete}
\bibinfo{author}{Hong, L.~J.}, \& \bibinfo{author}{Nelson, B.~L.}
  (\bibinfo{year}{2006}).
\newblock \bibinfo{title}{Discrete optimization via simulation using compass}.
\newblock {\it \bibinfo{journal}{Operations Research}\/},  {\it
  \bibinfo{volume}{54}\/}, \bibinfo{pages}{115--129}.
  \DOIprefix\doi{https://doi.org/10.1287/opre.1050.0237}.
\bibitem[{Hong et~al.(2010)Hong, Nelson \& Xu}]{hong2010speeding}
\bibinfo{author}{Hong, L.~J.}, \bibinfo{author}{Nelson, B.~L.}, \&
  \bibinfo{author}{Xu, J.} (\bibinfo{year}{2010}).
\newblock \bibinfo{title}{Speeding up compass for high-dimensional discrete
  optimization via simulation}.
\newblock {\it \bibinfo{journal}{Operations Research Letters}\/},  {\it
  \bibinfo{volume}{38}\/}, \bibinfo{pages}{550--555}.
  \DOIprefix\doi{https://doi.org/10.1016/j.orl.2010.09.003}.
\bibitem[{Hong et~al.(2015)Hong, Nelson \& Xu}]{hong2015discrete}
\bibinfo{author}{Hong, L.~J.}, \bibinfo{author}{Nelson, B.~L.}, \&
  \bibinfo{author}{Xu, J.} (\bibinfo{year}{2015}).
\newblock \bibinfo{title}{Discrete optimization via simulation}.
\newblock In {\it \bibinfo{booktitle}{Handbook of Simulation Optimization}\/}
  (pp. \bibinfo{pages}{9--44}).
\newblock \bibinfo{address}{New York, NY}: \bibinfo{publisher}{Springer}.
\newblock \DOIprefix\doi{https://doi.org/10.1007/978-1-4939-1384-8_2}.
\bibitem[{Hu et~al.(2008)Hu, Fu \& Marcus}]{hu2008model}
\bibinfo{author}{Hu, J.}, \bibinfo{author}{Fu, M.~C.}, \&
  \bibinfo{author}{Marcus, S.~I.} (\bibinfo{year}{2008}).
\newblock \bibinfo{title}{A model reference adaptive search method for
  stochastic global optimization}.
\newblock {\it \bibinfo{journal}{Communications in Information \& Systems}\/},
  {\it \bibinfo{volume}{8}\/}, \bibinfo{pages}{245--276}.
  \DOIprefix\doi{https://dx.doi.org/10.4310/CIS.2008.v8.n3.a4}.
\bibitem[{Jin \& Branke(2005)}]{jin2005evolutionary}
\bibinfo{author}{Jin, Y.}, \& \bibinfo{author}{Branke, J.}
  (\bibinfo{year}{2005}).
\newblock \bibinfo{title}{Evolutionary optimization in uncertain
  environments---a survey}.
\newblock {\it \bibinfo{journal}{IEEE Transactions on Evolutionary
  Computation}\/},  {\it \bibinfo{volume}{9}\/}, \bibinfo{pages}{303--317}.
  \DOIprefix\doi{https://doi.org/10.1109/TEVC.2005.846356}.
\bibitem[{Kaufmann et~al.(2016)Kaufmann, Capp{\'e} \&
  Garivier}]{kaufmann2016complexity}
\bibinfo{author}{Kaufmann, E.}, \bibinfo{author}{Capp{\'e}, O.}, \&
  \bibinfo{author}{Garivier, A.} (\bibinfo{year}{2016}).
\newblock \bibinfo{title}{On the complexity of best-arm identification in
  multi-armed bandit models}.
\newblock {\it \bibinfo{journal}{Journal of Machine Learning Research}\/},
  {\it \bibinfo{volume}{17}\/}, \bibinfo{pages}{1--42}.
\bibitem[{Koenig \& Law(1985)}]{koenig1985procedure}
\bibinfo{author}{Koenig, L.~W.}, \& \bibinfo{author}{Law, A.~M.}
  (\bibinfo{year}{1985}).
\newblock \bibinfo{title}{A procedure for selecting a subset of size m
  containing the l best of k independent normal populations, with applications
  to simulation}.
\newblock {\it \bibinfo{journal}{Communications in Statistics-Simulation and
  Computation}\/},  {\it \bibinfo{volume}{14}\/}, \bibinfo{pages}{719--734}.
  \DOIprefix\doi{https://doi.org/10.1080/03610918508812467}.
\bibitem[{Lattimore \& Szepesv{\'a}ri(2020)}]{lattimore2020bandit}
\bibinfo{author}{Lattimore, T.}, \& \bibinfo{author}{Szepesv{\'a}ri, C.}
  (\bibinfo{year}{2020}).
\newblock {\it \bibinfo{title}{Bandit Algorithms}\/}.
\newblock \bibinfo{address}{Cambridge, UK}: \bibinfo{publisher}{Cambridge
  University Press}.
\bibitem[{Liu et~al.(2017)Liu, P{\'e}rez-Li{\'e}bana \& Lucas}]{liu2017bandit}
\bibinfo{author}{Liu, J.}, \bibinfo{author}{P{\'e}rez-Li{\'e}bana, D.}, \&
  \bibinfo{author}{Lucas, S.~M.} (\bibinfo{year}{2017}).
\newblock \bibinfo{title}{Bandit-based random mutation hill-climbing}.
\newblock In {\it \bibinfo{booktitle}{2017 IEEE Congress on Evolutionary
  Computation}\/} (pp. \bibinfo{pages}{2145--2151}).
\newblock \bibinfo{organization}{IEEE}.
\newblock \DOIprefix\doi{10.1109/CEC.2017.7969564}.
\bibitem[{Lucas et~al.(2018)Lucas, Liu \& Perez-Liebana}]{lucas2018n}
\bibinfo{author}{Lucas, S.~M.}, \bibinfo{author}{Liu, J.}, \&
  \bibinfo{author}{Perez-Liebana, D.} (\bibinfo{year}{2018}).
\newblock \bibinfo{title}{The n-tuple bandit evolutionary algorithm for game
  agent optimisation}.
\newblock In {\it \bibinfo{booktitle}{2018 IEEE Congress on Evolutionary
  Computation}\/} (pp. \bibinfo{pages}{1--9}).
\newblock \bibinfo{organization}{IEEE}.
\newblock \DOIprefix\doi{10.1109/CEC.2018.8477869}.
\bibitem[{Neto et~al.(2020)Neto, Constantino, Martins \&
  Pedroso}]{neto2020multi}
\bibinfo{author}{Neto, T.}, \bibinfo{author}{Constantino, M.},
  \bibinfo{author}{Martins, I.}, \& \bibinfo{author}{Pedroso, J.~P.}
  (\bibinfo{year}{2020}).
\newblock \bibinfo{title}{A multi-objective monte carlo tree search for forest
  harvest scheduling}.
\newblock {\it \bibinfo{journal}{European Journal of Operational Research}\/},
  {\it \bibinfo{volume}{282}\/}, \bibinfo{pages}{1115--1126}.
  \DOIprefix\doi{https://doi.org/10.1016/j.ejor.2019.09.034}.
\bibitem[{Pedroso \& Rei(2015)}]{pedroso2015tree}
\bibinfo{author}{Pedroso, J.~P.}, \& \bibinfo{author}{Rei, R.}
  (\bibinfo{year}{2015}).
\newblock \bibinfo{title}{Tree search and simulation}.
\newblock In {\it \bibinfo{booktitle}{Applied Simulation and Optimization}\/}
  (pp. \bibinfo{pages}{109--131}).
\newblock \bibinfo{address}{Cham, Switzerland}: \bibinfo{publisher}{Springer}.
\newblock \DOIprefix\doi{https://doi.org/10.1007/978-3-319-15033-8_4}.
\bibitem[{Qiu \& Miikkulainen(2019)}]{qiu2019enhancing}
\bibinfo{author}{Qiu, X.}, \& \bibinfo{author}{Miikkulainen, R.}
  (\bibinfo{year}{2019}).
\newblock \bibinfo{title}{Enhancing evolutionary conversion rate optimization
  via multi-armed bandit algorithms}.
\newblock In {\it \bibinfo{booktitle}{Proceedings of the AAAI Conference on
  Artificial Intelligence}\/} (pp. \bibinfo{pages}{9581--9588}).
\newblock volume~\bibinfo{volume}{33}.
\newblock \DOIprefix\doi{https://doi.org/10.1609/aaai.v33i01.33019581}.
\bibitem[{Salemi et~al.(2019)Salemi, Song, Nelson \& Staum}]{l2019gaussian}
\bibinfo{author}{Salemi, P.~L.}, \bibinfo{author}{Song, E.},
  \bibinfo{author}{Nelson, B.~L.}, \& \bibinfo{author}{Staum, J.}
  (\bibinfo{year}{2019}).
\newblock \bibinfo{title}{Gaussian markov random fields for discrete
  optimization via simulation: Framework and algorithms}.
\newblock {\it \bibinfo{journal}{Operations Research}\/},  {\it
  \bibinfo{volume}{67}\/}, \bibinfo{pages}{250--266}.
  \DOIprefix\doi{https://doi.org/10.1287/opre.2018.1778}.
\bibitem[{Semelhago et~al.(2021)Semelhago, Nelson, Song \&
  W{\"a}chter}]{semelhago2021rapid}
\bibinfo{author}{Semelhago, M.}, \bibinfo{author}{Nelson, B.~L.},
  \bibinfo{author}{Song, E.}, \& \bibinfo{author}{W{\"a}chter, A.}
  (\bibinfo{year}{2021}).
\newblock \bibinfo{title}{Rapid discrete optimization via simulation with
  gaussian markov random fields}.
\newblock {\it \bibinfo{journal}{INFORMS Journal on Computing}\/},  {\it
  \bibinfo{volume}{33}\/}, \bibinfo{pages}{915--930}.
  \DOIprefix\doi{https://doi.org/10.1287/ijoc.2020.0971}.
\bibitem[{Shi \& {\'O}lafsson(2000)}]{shi2000nested}
\bibinfo{author}{Shi, L.}, \& \bibinfo{author}{{\'O}lafsson, S.}
  (\bibinfo{year}{2000}).
\newblock \bibinfo{title}{Nested partitions method for stochastic
  optimization}.
\newblock {\it \bibinfo{journal}{Methodology and Computing in Applied
  Probability}\/},  {\it \bibinfo{volume}{2}\/}, \bibinfo{pages}{271--291}.
  \DOIprefix\doi{https://doi.org/10.1023/A:1010081212560}.
\bibitem[{Sivanandam \& Deepa(2008)}]{sivanandam2008genetic}
\bibinfo{author}{Sivanandam, S.~N.}, \& \bibinfo{author}{Deepa, S.~N.}
  (\bibinfo{year}{2008}).
\newblock \bibinfo{title}{Genetic algorithms}.
\newblock In {\it \bibinfo{booktitle}{Introduction to Genetic Algorithms}\/}.
\newblock \bibinfo{address}{Berlin, Germany}: \bibinfo{publisher}{Springer}.
\newblock \DOIprefix\doi{https://doi.org/10.1007/978-3-540-73190-0}.
\bibitem[{Sun et~al.(2014)Sun, Hong \& Hu}]{sun2014balancing}
\bibinfo{author}{Sun, L.}, \bibinfo{author}{Hong, L.~J.}, \&
  \bibinfo{author}{Hu, Z.} (\bibinfo{year}{2014}).
\newblock \bibinfo{title}{Balancing exploitation and exploration in discrete
  optimization via simulation through a gaussian process-based search}.
\newblock {\it \bibinfo{journal}{Operations Research}\/},  {\it
  \bibinfo{volume}{62}\/}, \bibinfo{pages}{1416--1438}.
  \DOIprefix\doi{https://doi.org/10.1287/opre.2014.1315}.
\bibitem[{Wang et~al.(2013)Wang, Pasupathy \& Schmeiser}]{wang2013integer}
\bibinfo{author}{Wang, H.}, \bibinfo{author}{Pasupathy, R.}, \&
  \bibinfo{author}{Schmeiser, B.~W.} (\bibinfo{year}{2013}).
\newblock \bibinfo{title}{Integer-ordered simulation optimization using
  r-spline: Retrospective search with piecewise-linear interpolation and
  neighborhood enumeration}.
\newblock {\it \bibinfo{journal}{ACM Transactions on Modeling and Computer
  Simulation}\/},  {\it \bibinfo{volume}{23}\/}, \bibinfo{pages}{1--24}.
  \DOIprefix\doi{https://doi.org/10.1145/2499913.2499916}.
\bibitem[{Xu et~al.(2013)Xu, Nelson \& Hong}]{xu2013adaptive}
\bibinfo{author}{Xu, J.}, \bibinfo{author}{Nelson, B.~L.}, \&
  \bibinfo{author}{Hong, L.~J.} (\bibinfo{year}{2013}).
\newblock \bibinfo{title}{An adaptive hyperbox algorithm for high-dimensional
  discrete optimization via simulation problems}.
\newblock {\it \bibinfo{journal}{INFORMS Journal on Computing}\/},  {\it
  \bibinfo{volume}{25}\/}, \bibinfo{pages}{133--146}.
  \DOIprefix\doi{https://doi.org/10.1287/ijoc.1110.0481}.
\bibitem[{Yan \& Mukai(1992)}]{yan1992stochastic}
\bibinfo{author}{Yan, D.}, \& \bibinfo{author}{Mukai, H.}
  (\bibinfo{year}{1992}).
\newblock \bibinfo{title}{Stochastic discrete optimization}.
\newblock {\it \bibinfo{journal}{SIAM Journal on Control and Optimization}\/},
  {\it \bibinfo{volume}{30}\/}, \bibinfo{pages}{594--612}.
  \DOIprefix\doi{https://doi.org/10.1137/0330034}.

\end{thebibliography}






\end{document}